%% file: main.tex
\documentclass[sigconf]{acmart}
\input{macro}

\AtBeginDocument{%
  }

\copyrightyear{2026}
\acmYear{2026}
\setcopyright{cc}
\setcctype{by}
\acmConference[KDD '26]{Proceedings of the 32nd ACM SIGKDD Conference on Knowledge Discovery and Data Mining V.2}{August 09--13, 2026}{Jeju Island, Republic of Korea}
\acmBooktitle{Proceedings of the 32nd ACM SIGKDD Conference on Knowledge Discovery and Data Mining V.2 (KDD '26), August 09--13, 2026, Jeju Island, Republic of Korea}
\acmDOI{10.1145/3770855.3817658}
\acmISBN{979-8-4007-2259-2/2026/08}

\begin{document}

\title{NBQ: Next-Best-Question for Dynamic Profiling}


\author{Yimin Shi}
\email{yiminshi@comp.nus.edu.sg}
\affiliation{%
  \institution{National University of Singapore}
  \city{Singapore}
  \country{Singapore}
}

\author{Clarice Wang}
\email{clarice7@seas.upenn.edu}
\affiliation{%
  \institution{University of Pennsylvania}
  \city{Philadelphia}
  \state{PA}
  \country{USA}
}

\author{Haixun Wang}
\email{haixun@gmail.com}
\affiliation{%
  \institution{EvenUp}
  \city{San Francisco}
  \state{CA}
  \country{USA}
}

\author{Xiaokui Xiao}
\email{xkxiao@nus.edu.sg}
\affiliation{%
  \institution{National University of Singapore}
  \city{Singapore}
  \country{Singapore}
}

\renewcommand{\shortauthors}{Yimin Shi, Clarice Wang, Haixun Wang, and Xiaokui Xiao}

    
    
    
    
    
    
    
    
    


\begin{abstract}
Many real-world conversational settings for knowledge discovery, including podcasts, hiring screens, and marketplaces, need a purpose-driven understanding of a person. For this, the most effective path is to ask the Next Best Question (NBQ) at each turn: the question with the highest expected information gain given what’s already been learned and the conversation’s goal. We propose the \ours framework, a plug-and-play smart asker that, given any topic, seeds a diverse pool of hundreds to thousands of candidate questions, maintains a compact, continuously updated user state that tracks coverage and confidence, selects the next question after every answer to maximize incremental value within the pre-specified turn budget,
and finally distills the unstructured Q\&A dialogue into a structured, vector-based user profile ready for downstream mining tasks.
As a demanding application, we instantiate \ours for \emph{reciprocal matchmaking}, where ``reciprocal'' means compatibility must be mutual,
\ie a match of two people is valid only when one fits the other’s preferences, and vice versa.
Therefore, each person is modeled with two representations: (i) who they are (self-description) and (ii) whom they prefer (counterpart preferences). \ours asks questions to refine both vectors and thereby increase the match success probability.
To scale matching to real-world platforms with millions of concurrent users and continuously updated profiles, we introduce \ourmatch, an efficient retrieval layer that recasts reciprocal matching from quadratic pairwise scoring to approximate vector search. With modest storage overhead, \ourmatch updates each user’s top matches in real time.
Compared with random or conventional generative questioning baselines, \ours improves the quality of user profiling by up to 13.6\% and 14.0\% in terms of AC@T and AR@T. Meanwhile, \ourmatch accelerates retrieval by up to 22.9$\times$ while maintaining a high recall of up to 0.989.

\end{abstract}





\maketitle
\newcommand\kddavailabilityurl{https://doi.org/10.5281/zenodo.20504410}
\ifdefempty{\kddavailabilityurl}{}{
\begingroup\small\noindent\raggedright\textbf{Resource Availability:}\\
The source code of this paper has been made publicly available at \url{\kddavailabilityurl}.
\endgroup
}

\input{sections/nbq-problem}
\input{sections/agent}
\input{sections/matching}

\input{sections/exp}
\input{sections/related}
\input{sections/conclusion}

\begin{acks}
This work was supported by the National Research Foundation, Singapore
and Infocomm Media Development Authority under its Trust Tech Funding
Initiative.
\end{acks}

\bibliographystyle{ACM-Reference-Format}
\balance
\bibliography{ref}

\input{sections/appendix}

\end{document}

%% file: macro.tex
\usepackage{graphicx}
\usepackage[ruled,vlined,linesnumbered]{algorithm2e}
\usepackage[noend]{algpseudocode}
\usepackage{amsmath}
\usepackage{xcolor}
\usepackage{xspace}
\usepackage{enumitem}
\usepackage{multirow}
\usepackage{makecell}
\usepackage{tabularx}
\usepackage{balance}
\hypersetup{
  colorlinks = true,
  linkcolor = black,
  citecolor = black,
  urlcolor  = blue
}

\usepackage{tikz}
\usepackage{pgfplots}
\pgfplotsset{compat=1.18}
\usepgfplotslibrary{groupplots}
\pgfplotsset{compat=1.18}
\usetikzlibrary{plotmarks}

\usepackage{adjustbox}

\newtheorem{theorem}{Theorem}

\newcommand{\ie}{{i.e.},\xspace}
\newcommand{\eg}{{e.g.},\xspace}

\newcommand{\stitle}[1]{\noindent{\bf #1.\/}}
\newcommand{\revise}[1]{#1}
\newcommand{\revisecolor}{\color{black}}

\newcommand{\todo}[1]{\textcolor{blue}{[\textbf{TODO:} #1]}}
\newcommand{\yimin}[1]{\textcolor{blue}{[\textbf{Yimin}: #1]}}

\newcommand{\ours}{\texttt{NBQ}\xspace}
\newcommand{\randask}{\texttt{RandAsk}\xspace}
\newcommand{\oursub}{\texttt{SubdimDiv}\xspace}

\newcommand{\ourmatch}{\texttt{QuickMatch}\xspace}
\newcommand{\man}{\textbf{ROMANTIC-MAN}}
\newcommand{\shortman}{\textbf{ROM-M}}
\newcommand{\woman}{\textbf{ROMANTIC-WOMAN}}
\newcommand{\shortwoman}{\textbf{ROM-W}}
\newcommand{\simplelsh}{\texttt{SIMPLE-LSH}\xspace}

\newcommand{\scenario}{\alpha \xspace}
\newcommand{\rel}{r}
\newcommand{\eat}[1]{} 

\newcommand{\haixun}[1]{\textcolor{red}{[\textbf{Haixun:} #1]}}

%% file: sections/nbq-problem.tex
\section{Introduction}

Interviews {for conversational knowledge discovery} are pervasive in communication: a podcast host probing for untold stories; a recruiter estimating job fit under a tight turn budget; a clinician triaging symptoms; a sales rep running discovery; a tutor diagnosing misconceptions; a journalist verifying claims; a UX researcher mapping needs; a concierge teasing out preferences. These workflows succeed or stall on the same fulcrum: \emph{what we ask next}.
At each turn of a conversation, given partial answers, time constraints, and a concrete purpose, the interviewer must choose the \emph{Next Best Question} (NBQ) {for profiling}. This is inherently a \emph{sequential} and \emph{ambiguous} decision problem: we adapt a sequence of questions, one step at a time, using what we have learned so far to decide where to probe next. Unlike answer-oriented assistants that primarily react, an NBQ agent \emph{leads}, steering the dialogue toward a goal.

A conventional questionnaire-based solution
struggles because
user answers can be inherently open, random, incomplete, or even off-topic. With limited user patience and a tight turn budget, static forms can waste turns on questions with lower values and miss chances to resolve the real uncertainty. Therefore, the interviewer agent needs an adaptive questioning policy that selects the next question that is expected to maximize the terminal utility under a strict interview budget.

\revise{Existing approaches, however, fall short of this goal in different ways.}
Prior conversation planning methods learn dialogue policies that optimize long-horizon conversational task rewards.
For example, HiCR~\cite{www2023hicr} learns a hierarchical policy for a conversational recommender system and SCOPE~\cite{chen2025scope} learns Markov decision process based on quantifiable metrics.
However, these methods rely on task-specific supervision signals that are often hard to obtain in practice, especially when the interview target itself is qualitative and ambiguous.
Recent user-elicitation methods prompt Large Language Models (LLMs) to ask questions under descriptive requirements. Specifically, MockLLM~\cite{sun2025mockllm} and GATE~\cite{lieliciting} specify constraints such as avoiding repetition and favoring questions that reveal the most unknown information. Although these methods pursue a related objective, they do not explicitly model how each question and user answer contribute to the conversation target, nor do they assume a dialogue budget. Moreover, they do not provide a clear path to convert the free-form Q\&A dialogue into an operable representation for downstream use.
Taken together, there is a gap: we lack a purpose-driven, budget-aware, and plug-and-play adaptive questioning interviewer that explicitly models uncertainty in partial and stochastic user answers, reasons about the relevance between historical answers and future questions, and produces structured dialogue representations for downstream applications.

This gap motivates our design of \ours, an \eat{\emph{AI interviewer}}\emph{AI-driven interviewing framework} that \emph{leads} the conversation. Unlike customer-support or QA assistants that react to human queries, \ours proactively decides what to ask next to drive the dialogue toward a declared purpose. For a given topic or interviewee and an objective $\alpha$, \ours: (i) synthesizes a comprehensive, diverse pool of candidate questions tailored to the topic/person; (ii) maintains a compact, continuously updated state that tracks both \emph{coverage} (what has been asked or implied) and \emph{confidence} (how well we know it); (iii) selects, at each turn within the budget, the next question that maximizes expected \emph{incremental value} toward $\alpha$ given the dialogue so far;
and (iv) distills the conversation into a structured profile (\eg a vector) for downstream use.
{\ours is designed to be domain-agnostic and plug-and-play, i.e., it can work with any topic by instantiating the above components.}

As a demanding application, we focus on \emph{reciprocal matchmaking}, where success requires \emph{mutual} satisfaction: who a person \emph{is} must align with what another person \emph{seeks}, and vice versa. This setting surfaces three key design challenges:

\stitle{Purpose-driven questioning} The value of a question is relative to the interview’s goal
(\eg  predictive validity in hiring, personality alignment in dating).
\ours must evaluate the \emph{expected gain} of each candidate question against that goal, given the conversation so far.

\stitle{Representation with uncertainty}
Open answers from the respondent can carry overlapping signals; one response can inform multiple dimensions. \ours must propagate credit to correlated attributes, track residual uncertainty, and favor questions that shrink uncertainty where it most affects the objective.

\stitle{Real-time, large-scale retrieval} Each answer should immediately sharpen both self and preference profiles and refresh mutual-compatible candidates fast enough to keep the interview fluid, even at population scale (hundreds of thousands to millions of users).

In this paper, we formalize NBQ as a sequential, purpose-driven interview design problem, beginning with a precise \emph{problem statement}.
We then describe the framework of our solution in Section~\ref{sec:overview} and leave the detailed algorithm discussion in Section~\ref{sec:adaptive}.
We then instantiate the framework for \emph{reciprocal matchmaking} in Section~\ref{sec:match}, introducing a user representation method which distills free-form dialogue to vectors, and a mutual-compatibility objective that proactively considers the bidirectional satisfaction.
{To avoid quadratic pairwise matching, we propose \ourmatch, which relaxes reciprocal compatibility into a tight pre-filtering condition that is implementable via approximate vector search, followed by an exact final check.}
Notably, compared to random and user-elicitation baselines, \ours can consistently improve profiling quality with respect to AC@T and AR@T by up to 13.6\% and 14.0\%, respectively. Meanwhile, \ourmatch accelerates candidate retrieval by up to 22.9$\times$ while maintaining a high recall of up to 0.989, confirming the framework's effectiveness and efficiency.

To summarize, we make the following main contributions:
\begin{itemize}[topsep=2pt,itemsep=1pt,parsep=0pt,partopsep=0pt,leftmargin=11pt]
    \item We propose \ours, which generates a comprehensive question set, and adaptively selects the next-best-question to maximize the expected incremental information gain.
    \item We present a transformation-based, task-agnostic user representation that distills profiles from unstructured dialogue.
    \item We propose \ourmatch, which scales reciprocal matching via efficient pre-filtering using approximate vector range search.
    \item We conduct experiments across scenarios and perspectives to demonstrate the effectiveness and efficiency of key components.
\end{itemize}
All source codes, prompts, and datasets are available at: \textcolor{blue}{\url{https://github.com/Hanc1999/Next-Best-Question}}.

\begin{figure*}[t]
    \centering
    \includegraphics[width=0.98\textwidth]{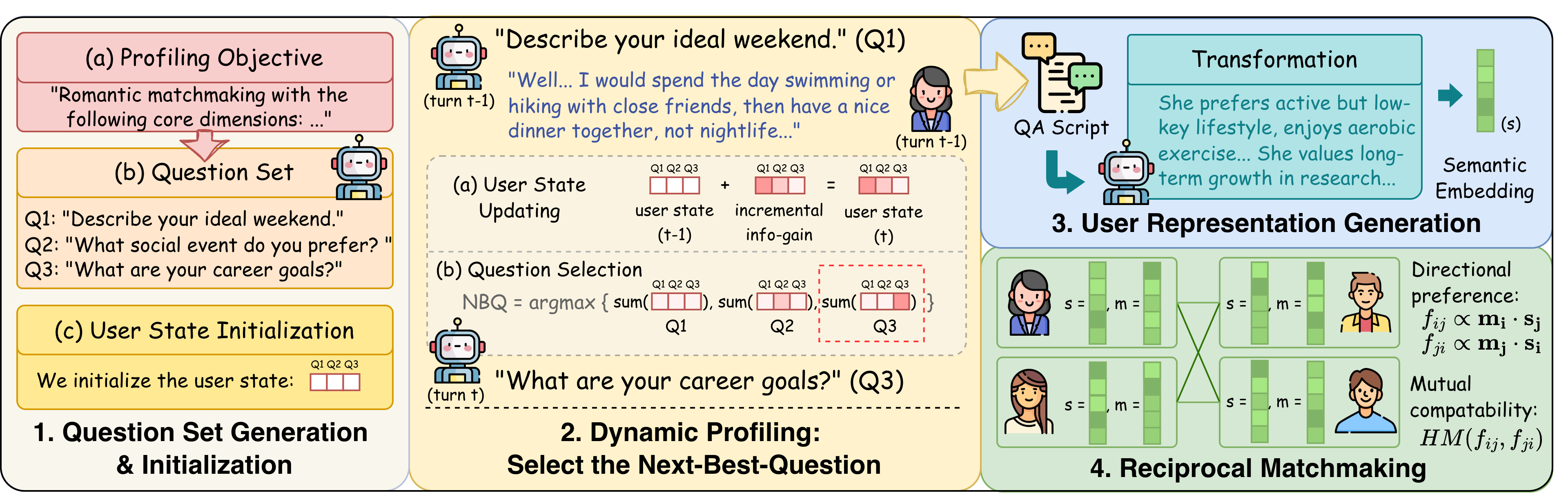}
    \caption{\revise{\ours workflow: question-pool construction, adaptive profiling, user representation, and reciprocal matching.}}
    \Description{Workflow diagram of NBQ showing question-pool construction, adaptive question selection with answer analysis and user-state updates, dialogue-to-profile transformation, and QuickMatch-based reciprocal matching.}
    \label{fig:kdd-workflow}
\end{figure*}

\section{Problem Statement}\label{sec:problem}

We formalize three core problems that the paper addresses: \emph{question-pool construction}, \emph{adaptive next-question selection}, and for the matchmaking application, 
\emph{reciprocal profiling and scalable match retrieval}.

\stitle{Question-Pool Construction}
Given a topic or scenario $\alpha$ (e.g., hiring, podcasting, clinical intake), the problem of question-pool construction asks for a set of candidate questions
$\mathcal{Q}(\alpha)=\{q_1,\ldots,q_n\}$
that covers a set of relevant sub-dimensions \eat{$\mathcal{D}$} of scenario $\alpha$. We require the set $\mathcal{Q}(\alpha)$ to satisfy certain predefined metrics such as coverage, diversity, and quality. 
These metrics are intended to be computable from question and answer representations.

\stitle{Adaptive Next-Best Question Selection}
An interview with respondent $u$ proceeds for at most $T$ turns. After $t \le T$ turns, the dialogue prefix is
$D_t=\{(q_1,a_1),\ldots,(q_t,a_t)\}.$
Let $J(D_t;\alpha)\in\mathbb{R}$ be an abstract utility function that scores the usefulness of the conversation so far for a topic $\alpha$ (e.g., narrative value, decision accuracy, risk reduction).
The problem here is to choose the next best question that is expected to maximize terminal utility under a strict turn budget $T$:
$
q_{t+1} = \arg\max_{q\in \mathcal{Q}} \mathbb{E}_q [J(D_T; \alpha)],\label{eq:global}
$
where the expectation is over the answer randomness.

\stitle{Efficient Reciprocal Matchmaking}
In reciprocal matching, each user $u$ is profiled along two axes: a \emph{self-description} ($s$) representation describing who they are, and a \emph{match-preference} ($m$) representation describing what they seek in a partner. Let $s$ and $m$ be their representations distilled from the dialogue $D$, we define a \emph{mutual compatibility} function 
$g: ({s}_i, {m}_i, {s}_j, {m}_j) \mapsto [0, 1],$
which takes the dual representations of two users $u_i$ and $u_j$ and returns a normalized matching score. The function $g$ is required to be \emph{reciprocal} (i.e., $g(u_i,u_j)=g(u_j,u_i))$, \emph{bounded} (i.e., $g \in[0,1]$), and \emph{monotone} in alignment.
Our objective is to maintain and expose, for each user $u_i$, the set of top-matching candidates which maximizes the overall utility
$
\textstyle \sum_{i \in \mathcal{U}}~\sum_{j \in \mathrm{TopK}(i)} g(u_i, u_j)\,
$
in \emph{real time} as answers arrive and user profiles update. The system must satisfy three key constraints: 
(i) \textit{streaming updates}, the representations $s$ and $m$ are refreshed immediately upon each new $(q,a)$ pair; (ii) \textit{low latency}, the ranked set $\mathrm{TopK}_t(i)$ must be refreshed within a fixed time bound, independent of the total user population; and (iii) \textit{scalability}, both space and computation per refresh should remain sublinear in $|\mathcal{U}|$. The specific design of $g$ and the underlying retrieval structures are detailed in subsequent sections.
{
These requirements pose great challenges for data management and determine whether the delivered system is scalable enough to operate in real-world environments with millions of users online at the same time.
}

%% file: sections/agent.tex
\section{Method Overview}\label{sec:overview}

\stitle{Dynamic Question Selection}
\revise{We use $\alpha$ to denote the \emph{conversational profiling objective}, \ie the topic or goal that drives the interview (\eg professional experience in hiring, partner expectation in matchmaking).}
We first define a universal candidate question pool $\mathcal{Q}(\alpha)$\revise{, the set of candidate questions tailored to $\alpha$,} that is shared across all users within the same scenario $\alpha$. In practice, the pool contains approximately hundreds to thousands of unique and diverse questions, carefully designed to cover all relevant information we aim to collect from users.
To quantify how well the system understands each user at any point during the conversation, we introduce the concept of a \textit{user state}, formally defined as a vector $\mathbf{s} \in \mathbb{R}^{|\mathcal{Q}|}$. Each entry $\mathbf{s}[i] \in [0,1]$ corresponds to the degree of completeness, or confidence, that question $q_i$ has been addressed. Initially, when no questions have been asked, we know nothing about the user; hence, all entries of $s$ are initialized to zero.
\eat{\haixun{Should it be zero, or should it be the ``default'' person value? But it's more difficult to decide what is the ``default'' value. We can think about this more.}} 
In an ideal scenario, where a user responds accurately to every question in $\mathcal{Q}$, all entries would approach one, indicating a complete understanding of the user.
However, in practice, it is unrealistic to ask every question due to constraints on conversation time and user patience. Therefore, we limit the turn budget as $T \ll |\mathcal{Q}|$, requiring the agent to strategically select which questions to pose within these $T$ turns to maximize the total \textit{information gain}, 
the overall conversation utility,
defined as the sum of the entries in the user state vector:
$\mathrm{\mathrm{IG}} = \sum_{i=1}^{|\mathcal{Q}|} \mathbf{s}[i].$

A critical insight guiding our design is that questions in $\mathcal{Q}$ often exhibit correlations; that is, answers provided by a user may simultaneously address multiple related questions to varying degrees. For instance, answers to \textit{``Describe your ideal weekend,''} often help to reveal the respondent's \textit{social preferences} and \textit{lifestyle choices}, which also inform correlated questions such as ``\textit{Who do you like to be with?}'' and ``\textit{Do you like being outdoors?}''.
Leveraging this, we not only update the user state entry corresponding to the asked question but also incrementally update the entries of other associated questions.
However, when answering the same question, different users will respond from varying perspectives shaped by their unique personal experiences and values.
As a result, it is infeasible to pre-select a globally optimal subset of $T$ questions; we need to adaptively select the next best question at each turn from $\mathcal{Q}$ based on the user’s individual responses, dynamically adjusting our questioning strategy to maximize the overall conversation utility.

To address this challenge, we propose a greedy adaptive question selection algorithm, denoted as \ours. At each conversation turn $t$, it selects the next best question $q_{t+1}^* \in \mathcal{Q}$ that is expected to yield the largest incremental information gain:
\begin{equation}
\begin{aligned}
q_{t+1}^* &= \arg\max_{q \in \mathcal{Q}} \mathbb{E}_{q} \left[ \mathrm{IG}_{t+1} \right] - \mathrm{IG}_t \\
&=\arg\max_{q \in \mathcal{Q}} \mathbf{1}^\top(\mathbb{E}_q[\mathbf{s}_{t+1}] - \mathbf{s}_t)
\end{aligned}\label{eq:max}
\end{equation}

\eat{\haixun{This following section needs a better structure. While titled ``reciprocal matching'', it talks about too many things, including representation, matching function, transforming D to text then to represetnation and user uncertainty. But this is an overview section. We can leave it here for now, and working on section 4 and beyond first.}}
\stitle{Reciprocal Matching}
One of the most demanding downstream applications of conversational profiling is reciprocal matching, where the system needs to collect a user's profile along two axes: (i) self-description and (ii) preferences towards counterparts.
To achieve this, we construct two corresponding question sets and apply \ours for two $T$-turn adaptive profiling conversations. 
To leverage the resulting dialogues for matching while filtering out noise (\eg overly niche details, specific expression habits, response length), we transform each unstructured Q\&A transcript into a readable, concise, high-level user representation.
Formally, we represent each user as a structured pair $(s, m)$, where $s$ summarizes her self-characteristics and $m$ captures her desired traits of the potential counterpart. Both $s$ and $m$ preserve the semantic richness and are further encoded into a shared embedding space for better actionability.

The goal of the reciprocal matching system is to identify user pairs with high mutual compatibility. We present a symmetric compatibility function $g$ that proactively considers the bidirectionality of candidate satisfaction and penalizes the imbalanced situations.
Formally, given two users $u_i$ and $u_j$, we define a satisfaction function $f_{ij}=f(m_i, s_j)$ to measure how well user $u_j$'s characteristics align with user $u_i$'s preferences, vice versa, $f_{ji}=f(m_j, s_i)$ measures the satisfaction in the opposite direction. A successful match should maximize both of these scores.
\eat{\haixun{Do we assume each question/answer has the same importance to represent a person? Usually it's not. For example, one very important requirement for x is that his match must be super rich. How do we model this?}}
\eat{\yimin{Yes, I think it is realistic for users to assign different weights to various dimensions, and we need to consider this during matching. The current user representation does not explicitly consider this.}} 
\section{Adaptive Profiling}\label{sec:adaptive}
In this section, we present the design details of \ours, the adaptive profiling module based on relevance analysis, which operates through the following four main steps: (i) recursively construct the subdimension set $\mathcal{D}(\alpha)$ for scenario $\alpha$ and generate the question set $\mathcal{Q}$ correspondingly; (ii) for every pair of questions, pre-compute their relevance $\rel(q_i,q_j)$;
(iii) at each turn $t$, given a new answer $a_t$, compute its relevance to the historical answers $A_{t-1}$ and to all questions in $\mathcal{Q}$, and update the user state $\mathbf{s}$;
(iv) select, for the next turn, the question that maximizes the expected information gain.
\subsection{Question Set Generation}\label{sec:qgen}
The question set comprises two subsets: $\mathcal{Q}=\{\mathcal{Q}_s, \mathcal{Q}_m\}$, where $\mathcal{Q}_s$ targets the user's self-information and $\mathcal{Q}_m$ focuses on her match preference.
To ensure that the question set is sufficiently comprehensive to cover all topics relevant to the subsequent reciprocal recommendation scenario $\alpha$, we first recursively construct a subdimension tree based on a small seed set of \textit{root dimensions} $\mathcal{D}_{root}$.

\stitle{\oursub} We propose the \oursub algorithm to construct the subdimension tree $\mathcal{D}$ through an iterative top-down process. Starting from the seed root dimension layer $\mathcal{D}_{0}$, we expand the set over $L$ iterations. At each iteration $l$, every subdimension in the previous layer $\mathcal{D}_{l-1}$ is further decomposed into children subdimensions. More specifically, conditioned on the given application scenario $\scenario$, we leverage LLMs to generate $k_{\text{min}}$ to $k_{\text{max}}$ new subdimensions with finer granularity and approximately equal importance. We require the division to be natural, and the resulting subdimensions should well cover the semantic scope of their parent. This recursive expansion process naturally introduces different levels of granularity and supports hierarchical understanding across dimensions. After processing the last layer, the resulting new subdimensions compose the new subdimension tree layer $\mathcal{D}_l$. The algorithm returns the union of all these layers as the result,
\ie $\mathcal{D}=\{\mathcal{D}_0, \ldots, \mathcal{D}_L\}$.
Pseudocode and prompts are provided in Appendix~\ref{app:alg} and~\ref{app:prompts:1}.

\subsection{Relevance Prediction}\label{subsec:relpred}
A core insight in the design of \ours is that user responses often carry overlapping content relevant to multiple questions. To exploit this, we introduce \textit{relevance}, which quantifies how much the answer to one question can help infer information about another.
Formally, for any pair of questions $(q_i, q_j) \in \mathcal{Q}\times \mathcal{Q}$, we define their relevance $\rel(q_i, q_j) \in [0,1]$ as the degree of content overlap between the user responses they are expected to elicit.
Similarly, we define \(\rel(q_i, a_j)\) and \(\rel(a_i, a_j)\) as the relevance between questions and answers, and between answers themselves, respectively.
For simplicity of presentation, we focus on the question--question relevance prediction, with other forms following the same pipeline.

Recent works have shown that LLMs can effectively judge the relevance between a user query and retrieved content~\cite{thomas2024large, farzi2025criteria}.
Motivated by this, we leverage LLMs to reason and estimate the relevance between questions.
Following~\citet{xiong2023can} and \citet{tian2023just}, we define six distinct relevance levels, each followed by a descriptive definition and a quantized overlap rate.
As the output, the LLM is instructed to generate a certainty distribution over these levels that sums to 1.
We then convert this distribution into a single continuous relevance score by taking a certainty-weighted average across all six levels.
We provide full details in Appendix~\ref{app:prompts:3}.

However, directly applying LLMs to score relevance for every pair of questions is impractical. The number of pairwise comparisons grows quadratically with the number of questions, leading to unaffordable computational costs.
To address this, we propose a lightweight relevance predictor: it first analyzes the answer content each question is expected to elicit, and then learns a mapping from the similarity between the two analyses to a relevance score.
Specifically, for each question \(q \in \mathcal{Q}\), we use LLMs to predict the \emph{answer elements} that are most likely to appear in a user’s response, denoted by $\hat{a}$.
For example, for \textit{``Describe your ideal weekend,''} an LLM may predict answer elements such as \emph{social relationships}, \emph{weekend lifestyle}, and \emph{entertainment activities}.
We carefully design the prompt so that the predicted elements are \textit{non-overlapping}, \textit{balanced}, \textit{complete}, and \textit{self-contained}.
Prompts provided in Appendix~\ref{app:prompts:4}.
Next, we assume and learn a monotone mapping between the relevance score and the answer analysis similarity.
We construct the calibration dataset by uniformly sampling question pairs from $\mathcal{Q}$, and
for each pair $(q_i,q_j)$, we obtain both the relevance rating $\hat{r}_{ij}$ and the analysis similarity score
$\bar{r}_{ij}=\cos(\mathbf{\hat{a}}_i,\mathbf{\hat{a}}_j)$, which is the cosine similarity between the semantic embeddings of ${\hat{a}}_i$ and ${\hat{a}}_j$. 
To reduce variance caused by LLM stochasticity, we repeat both computations three times and average the results.
We then fit the following polynomial curve,
$f(\bar{r}) = \sum_{k=0}^{d} w_k \, \bar{r}^k$,
optimizing parameters $\mathbf{w}$ to minimize the squared error between $f(\bar{r_{ij}})$ and $\hat{r_{ij}}$ over all sampled question pairs in the calibration set.
At inference time, we predict the relevance between two questions $(q_i,q_j)$ as
$$
r_{ij} = \operatorname{clip}_{[0,1]}\!\left(f\left(\bar{r}_{ij}\right)\right) = \min\{1,\max\{0,f\left(\bar{r}_{ij}\right)\}\}.
$$
Since pairwise question relevance is used in the question selection process and does not depend on the respondent, we precompute $\mathbf{R} \in \mathbb{R}^{|\mathcal{Q}|\times|\mathcal{Q}|}$ with $\mathbf{R}[i,j]=\rel_{ij}$.
This shifts the main cost of \ours to an offline stage and amortizes it across all users.
\revise{We empirically validate this learned mapping in Section~\ref{sec:rq5-components}.}

\subsection{User State Updating}
The question state dynamically reflects the extent to which a question can be addressed based on the user’s current responses.
Initially, since we know nothing about the respondent, we set all the questions' scores to zero, \ie $\mathbf{s}_0[i]=0, \forall q_i\in \mathcal{Q}$.
At turn $t$, after receiving a new user response $a_t$, we compute the incremental information gain it contributes to each question $q_i$ as:
{$$\delta(a_t \mid q_i, A_{t-1})=\rel(a_t, q_i) -\rel\left(a_t, A_{t-1} \mid q_i\right),$$}
where $\rel(a_t, q_i)$ denotes the relevance between the answer and the question, and $\rel(a_t, A_{t-1} \mid q_i)$ represents the relevance between $a_t$ and the existing answer set $A_{t-1}$ with respect to $q_i$. 
Intuitively, a larger $\rel(a_t, q_i)$ means that $a_t$ is more informative for answering $q_i$, thus yielding higher incremental information gain.
The second term acts as a redundancy penalty: it deducts the portion of $a_t$’s contribution with respect to $q_i$ that is already covered by earlier conversation turns.
More specifically, to compute $\rel(a_t,q_i)$, we follow the same procedure as for question–question relevance. We first ask the LLM to summarize the main content elements included in $a_t$ into a text $\hat{a}_t$, and encode it into a semantic embedding $\mathbf{\hat{a}}_t$. We then map the cosine similarity between $\mathbf{\hat{a}}_t$ and $\mathbf{\hat{a}}_i$ through the calibration function $f$ and clip the result to $[0,1]$:
$
\rel(a_t,q_i)=\operatorname{clip}_{[0,1]}\!\left(f\left(\cos\left(\mathbf{\hat{a}}_t, \mathbf{\hat{a}}_i\right)\right)\right).
$

When calculating $\rel(a_t, A_{t-1} \mid q_i)$, a naive approach is to compute the average relevance between $a_t$ and each past answer $a_j, \forall a_j \in A_{t-1}$. However, this ignores the fact that different past answers may contribute unequally to the question $q_i$. To address this, we define the relevance between $a_t$ and the existing question-specific knowledge regarding $q_i$ as the following weighted sum:
$$
\rel(a_t, A_{t-1} \mid q_i)= \textstyle \sum\limits_{j=1}^{t-1} \left(\delta_i(a_j) \cdot \rel(a_t, a_j)\right),
$$
where $\delta_i(a_j)$ denotes the incremental information gain that the historical answer $a_j$ has contributed to question $q_i$, which is actually the previous redundancy-adjusted relevance between them. The term $\rel(a_t,a_j)$ denotes the relevance between the current answer $a_t$ and the previous answer $a_j$.
The relevance between two user answers, $\rel(a_t, a_j)$, is computed and calibrated in the same way as question–question and answer–question relevance:
$
\rel(a_t,a_j)=\operatorname{clip}_{[0,1]}\!\left(f\left(\cos\left(\mathbf{\hat a}_t, \mathbf{\hat a}_j\right)\right)\right).
$
Intuitively, this relevance measurement pays more attention to past answers that are important to the question $q_i$ while downweighting those with limited effect.

\subsection{Question selection}
When selecting the next question to ask, \ours applies a greedy policy that maximizes the expected overall information gain provided by the user’s next response. The optimization problem described in Eq.~\eqref{eq:max} is instantiated as follows:
\begin{align*}
q^*_{t+1} &= \arg\max_{q_i \in \mathcal{Q}} \mathbb{E}_{q_i} \left[ \sum_{q_j \in \mathcal{Q}} \delta(a_{t+1} \mid q_j, A_{t}) \right] \\
         &= \arg\max_{q_i \in \mathcal{Q}} \sum_{q_j \in \mathcal{Q}} \delta(\hat{a}_i \mid q_j, A_{t}),
\end{align*}
where we approximate the unknown next response $a_{t+1}$ with the predicted answer elements $\hat{a}_i$ for candidate question $q_i$.

We next formulate the incremental information gain computation for all question candidates in matrix form.
Recall that we pre-compute a symmetric relevance matrix $\mathbf{R}$ over questions. We use $\mathbf{R}$ to estimate the relevance between the next answer to each candidate question and all other questions.
Let $\mathbf{R}_a \in \mathbb{R}^{T \times |\mathcal{Q}|}$ denote the relevance between the $T$ historical answers and all candidate questions, where $\mathbf{R}_a[t,i]=\rel(a_t, q_i)$.
Let $\mathbf{\Delta} \in \mathbb{R}^{T \times |\mathcal{Q}|}$ denote the incremental information gain matrix between all historical answers and all questions, \ie $\mathbf{\Delta}[t,i] = \delta_j(a_t) = \delta(a_t | q_i, A_{t-1})$.
{
We then select the next question by maximizing the following objective function over all candidates:
$$
q^*_{t+1}
=\arg\max_{q_i\in \mathcal{Q}}
\left(\big(\mathbf{R}-\mathbf{R}_a^\top\mathbf{\Delta}\big)\mathbf{1}\right)[i],
$$
where $\mathbf{1}\in\mathbb{R}^{|\mathcal{Q}|}$ is the all-ones column vector, it sums each row of $\mathbf{R}-\mathbf{R}_a^\top\mathbf{\Delta}$, yielding for each candidate question the expected incremental information gain aggregated over all questions.

%% file: sections/matching.tex
\section{Matching Mechanism}\label{sec:match}
In this section, to support reciprocal matchmaking as the downstream application, we discuss (i) how we convert the user transcript into a structured and operable representation, (ii) the design of mutual compatibility scores, and (iii) \ourmatch, an approximate retrieval layer that scales the online matching.

\subsection{User Representation}\label{sec:user-rep}
{
The reciprocal recommendation needs to be performed in a mutual manner, where each user serves both as a candidate and as a seeker. We therefore represent each user $u$ as $(s, m)$, where $s$ describes the user’s self characteristics and $m$ describes the qualities she seeks in a partner.}
To obtain $s$ and $m$, \ours conducts two $T$-turn conversations with each user and outputs corresponding Q\&A transcripts.
A vanilla method is to directly represent $s$ and $m$ as their corresponding dialogue transcript. However, this approach introduces several sources of noise. First, the resulting presentation is sensitive to surface-form factors such as personal tone, phrasing style, or conversational habits, which are largely irrelevant to the reciprocal matchmaking target.
Second, the same high-level characteristic can be expressed in very different ways across users, leading to inconsistent representations. 
For example, users who favor ``movies about space travel'' or ``stories about future technologies'' indicate the same interest in sci-fi themes.

\revise{To overcome these noise sources}, we propose an LLM-based transformation that converts the raw conversation into a more concise, robust user profile, capturing the user’s essential traits and preferences. Specifically, we instruct the LLM to abstract away superficial linguistic details and focus on persona-level insights, with main operations as follows: \textit{Insight Extraction}, \textit{Key Signal Retention}, and \textit{Structured Expression}.
We provide the operation details and full prompts in Appendix~\ref{app:prompts:2}.
To enable quantitative matching, we encode texts $s$ and $m$ to semantic embedding vectors $\mathbf{s}$ and $\mathbf{m}$ using general transformer-based encoder models. They capture the user profile information in a numerical form, allowing direct comparison between users.

\subsection{Matching Criteria}
We then define a satisfaction score to quantify one-directional compatibility between two users. For a given pair of users $(u_i, u_j)$, let $f_{ij}\in [0,1]$ denote how well $u_j$ fits $u_i$’s preferences:
\begin{align*}
f_{ij} 
&= \frac{1}{2}\cos(\mathbf{m}_i, \mathbf{s}_j)+\frac{1}{2}  
= \frac{1}{2}\frac{\mathbf{m}_i \cdot \mathbf{s}_j}{\| \mathbf{m}_i \| \| \mathbf{s}_j \|} + \frac{1}{2}.
\end{align*}
A high $f_{ij}$ indicates that $u_j$ has the attributes that $u_i$ is looking for; conversely, a low $f_{ij}$ implies $u_j$ likely does not meet $u_i$’s criteria. Note that in general $f_{ij} \neq f_{ji}$, since one user can be satisfied with a second even if the second isn’t satisfied with the first, reflecting asymmetric preferences.
To obtain a symmetric measure of mutual compatibility, we applied the harmonic mean to aggregate the two directional scores $f_{ij}$ and $f_{ji}$: 
\begin{equation}\label{eq:g_def}
    g(u_i, u_j) = \frac{2\,f_{ij} \cdot f_{ji}}{f_{ij} + f_{ji}}
\end{equation}
where we define $g(u_i,u_j)=0$ when $f_{ij}+f_{ji}=0$.
The harmonic mean is particularly suitable here because it produces a high value only when both sides are jointly satisfied; if either $f_{ij}$ or $f_{ji}$ is small, $g(u_i, u_j)$ will be strongly penalized. This ensures that a high compatibility score arises only from truly reciprocal alignment, preventing matches driven by one-sided attraction.

\subsection{Fast Approximation}
Retrieving all user pairs whose reciprocal matching score $g(u_i,u_j)$ exceeds a threshold $\theta$ by comparing all candidate pairs requires quadratic computational complexity $O(|\mathcal{U}|^2)$, which is infeasible in real-world matching platforms with hundreds of thousands to millions of concurrent users.
To make retrieval scalable, we propose \ourmatch, a two-stage approximation method.
It (i) derives a tight \emph{necessary} condition for retrieval, (ii) pre-filters candidate pairs via approximate vector search, and (iii) exactly computes mutual compatibility for the remaining pairs.

\stitle{Efficient Pre-filtering}
We derive a necessary condition by relating the compatibility score to the dot product of two user embeddings concatenated in reverse order (detailed proofs in Appendix~\ref{app:a}). 

\begin{theorem}\label{thm:1}
Let $f_{ij},f_{ji}\in(0,1]$. If $g(u_i,u_j) \geq \theta$, then necessarily $[\bar{\mathbf{s}}_i;\,\bar{\mathbf{m}}_i] \cdot [\bar{\mathbf{m}}_j;\,\bar{\mathbf{s}}_j]
\;\ge\; 4\theta - 2$ where $\bar{\mathbf{m}} = \frac{\mathbf{m}}{\|\mathbf{m}\|}$ and $\bar{\mathbf{s}} = \frac{\mathbf{s}}{\|\mathbf{s}\|}.$
\end{theorem}

This motivates a practical user indexing strategy. For each user $u_i$, we store two normalized concatenations, $[\bar{\mathbf{s}}; \bar{\mathbf{m}}]$ and $[\bar{\mathbf{m}}; \bar{\mathbf{s}}]$.  
Given a query user $u_i$, we use $[\bar{\mathbf{s}}_i; \bar{\mathbf{m}}_i]$ as the query vector and compute inner products with the reversed representations of other users, \ie $[\bar{\mathbf{m}}_j; \bar{\mathbf{s}}_j]$ for all $u_j \in \mathcal{U} \setminus \{u_i\}$. By Theorem~\ref{thm:1}, selecting all candidate users whose inner products are at least $4\theta-2$ guarantees that we include every user $u_j$ satisfying $g(u_i,u_j)\ge \theta$.

{
We have now transformed the original filtering condition into a standard range search problem. Range search has been extensively studied, and its approximate versions can be effectively accelerated using vector indexing techniques such as IV\_FLAT~\cite{johnson2019billion}, IVF\_SQ8~\cite{johnson2019billion}, and HNSW~\cite{malkov2018efficient}. It is supported as a scalable primitive in modern vector databases such as Milvus, Redis, pgvector, GaussDB-Vector, AnalyticDB-V, FAISS, and PUCK~\cite{wei2020analyticdb, wang2021milvus, pan2024survey, douze2024faiss, sun2025gaussdb, redisvectordb, pgvector2021, puck2023}.
These systems can handle billion-scale vector search very efficiently~\cite{simhadri2022results}.
We therefore delegate the range search component in \ourmatch to an off-the-shelf vector database, which retrieves candidate pairs satisfying $[\bar{\mathbf{s}}_i;\,\bar{\mathbf{m}}_i] \cdot [\bar{\mathbf{m}}_j;\,\bar{\mathbf{s}}_j]
\;\ge\; 4\theta - 2$ using approximate search to balance effectiveness and practical cost.
After this pre-filtering step, we obtain a candidate set that satisfies the necessary condition.
We then perform exact verification on these candidates by computing $g(u_i,u_j)$, and retain precisely those pairs with $g(u_i,u_j)\ge \theta$.
}

\stitle{Gap Analysis}
The practical efficiency of \ourmatch depends on two factors: (i) the gap between the pre-filtered candidates and the true matches, \ie whether the necessary condition used for range search is sufficiently tight; and (ii) the efficiency of approximated vector range search.
Here, we focus on the first factor and analyze the theoretical error analysis of the gap between the necessary condition and the exact mutual-compatibility score.

\begin{theorem}\label{thm:2}
Let $f_{ij},f_{ji}\in(0,1]$ and $\theta\geq\frac{1}{2}$. If a user pair $(u_i, u_j)$ satisfies $[\bar{\mathbf{s}}_i;\,\bar{\mathbf{m}}_i] \cdot [\bar{\mathbf{m}}_j;\,\bar{\mathbf{s}}_j]
\;\ge\; 4\theta - 2$ then $g(u_i, u_j) \geq \theta - \frac{(1-\theta)^2}{\theta}.$
\end{theorem}

In Theorem~\ref{thm:2}, we assume $\theta \ge \tfrac{1}{2}$, since in practice we typically filter only pairs that are at least moderately compatible.
We derive the gap term as ${(1-\theta)^2}/{\theta}$, where the numerator $(1-\theta)^2$ decays quadratically as $\theta$ grows, while the denominator $\theta$ increases linearly. Hence, the gap shrinks to $0$ as $\theta \to 1$ and is essentially determined by $(1-\theta)^2$. Equivalently, letting $\theta=1-\epsilon$ gives the gap equals $\tfrac{\epsilon^2}{1-\epsilon} = \epsilon^2 + O(\epsilon^3)$.
For instance, when $\theta = 0.9$, the difference between the compatibility score of pre-filtering results and the target threshold is bounded by $(1 - 0.9)^2 / 0.9 \approx 0.0111$; when $\theta = 0.85$, the bound is $0.0264$. These results indicate that when the threshold $\theta$ is relatively high, the approximation error remains very small, suggesting that our necessary condition is tight and reliable in practical settings.

%% file: sections/exp.tex
\section{Experiments}
In this section, we evaluate the proposed \ours, focusing on its application to reciprocal matchmaking. We further assess the effectiveness of the transformation-based user representation, as well as the accuracy and efficiency of the proposed \ourmatch.
We aim to answer the following \revise{five} research questions:

\begin{itemize}
[topsep=2pt,itemsep=1pt,parsep=0pt,partopsep=0pt,leftmargin=11pt]
    \item \textbf{RQ1}: How effectively can \ours elicit respondent profiles?
    \item \textbf{RQ2}: How do different LLM backbones and semantic embedding models influence \ours’s performance?
    \item \textbf{RQ3}: How effectively does the transformation-based user representation improve matching accuracy?
    \item \textbf{RQ4}: How accurate and efficient is \ourmatch in retrieving high-scoring matches at scale?
    \item \textbf{RQ5}: \revise{How effective are NBQ components?}
\end{itemize}
Experiments are conducted on a Linux machine with 32 Intel(R) Core(TM) i9-14900KF and 62GB RAM.

\input{tabls/main_compare}

\subsection{Dataset Description}
\stitle{OpenMatch}
We use OpenBench~\cite{wang2025framework} benchmark for the user profiling evaluation. To focus on the adaptive setting and reciprocal matchmaking downstream, following the synthesis and evaluation pipeline of OpenBench, we further synthesize the OpenMatch respondent instances under two scenarios: romantic matchmaking and job interviews.
We synthesize 100 respondents for the romantic matchmaking scenario and 20 for the job interview scenario, each based on a self-information script consisting of 10 topics and a corresponding questionnaire with 10 questions, yielding 1,200 questions in total.
For the matching evaluation, we divide romantic matchmaking respondents into two subsets, ROMANTIC-MAN and ROMANTIC-WOMAN, each containing 50 instances. We then construct 50 distinct one-to-one true match pairs across the two respondent subsets. For each matched pair, we correspondingly synthesize mate-preference scripts based on each other’s self-information scripts.
Full details provided in Appendix~\ref{app:data}.


\stitle{RandUser-100K}  
To evaluate the scalability and accuracy of our proposed \ourmatch, we construct this synthetic dataset with 100,000 users. Each user is represented by a pair of embeddings $(\mathbf{s}, \mathbf{m})$, where both are uniformly randomized in a 16-dimensional space. Specifically, we sample each dimension from $[0,1]$ and normalize the vectors to unit length so that all embeddings lie on the same hypersphere.  
We choose a 16-dimensional space to ensure adequate coverage of the sphere without requiring an excessively large number of vectors, enabling efficient and reliable evaluation of large-scale matching performance.

\subsection{Experiment Settings}\label{sec:setup}
\stitle{Baselines}
To evaluate the performance of user profiling, we compare \ours with three groups of baselines as follows:
\begin{itemize}[topsep=2pt,itemsep=1pt,parsep=0pt,partopsep=0pt,leftmargin=11pt]
    \item OpenBench interview methods, in which we apply prompt engineering following the original framework~\cite{wang2025framework}, explicitly emphasize the main dimensions of interest in each conversation and clearly specify both the total dialogue-turn budget and the remaining turn quota at each step. Regarding the backbone models, we include three commercial LLMs (GPT-5, GPT-5-mini, and Gemini-2.5-Pro) and three open-source models (Qwen-3-80B, DeepSeek-V3.1, and Llama-3.3-70B). These models represent the latest generation of LLMs, all capable of multi-turn conversation and reasoning~\cite{wang2024mint, zhang2024probing}, thereby serving as competitive baselines.
    \item LLM-based user-preference elicitation methods: MockLLM~\cite{sun2025mockllm} and GATE~\cite{lieliciting}, using GPT-5 as the backbone. For MockLLM, to ensure fair alignment with the other baselines, the agent has no access to public information. For GATE, we use the \textit{generating open-ended questions} setting. 
    To ensure a fair comparison, we specify the main dimensions of interest for the conversation.
    \item \randask: An ablation variant that replaces \ours's adaptive question selection with random sampling.
\end{itemize}

\stitle{Experiment Configurations}
\revise{When constructing the question set for both scenarios, we initialize $\mathcal{D}_{0}$ with the main dimensions of interest used to construct OpenMatch. For each dimension in $\mathcal{D}_{0}$, we generate 10 questions, yielding a question set with $|\mathcal{Q}|=100$.}
We apply GPT-4o to generate relevance ratings and apply GPT-4.1 for answer element analysis, both average three independent runs. We average results over three independent runs and then fit a cubic polynomial curve for calibration.
We use OpenAI's text-embedding-3-large (O-3-large) as the default semantic encoding model with a dimension of 1024.
Unless stated otherwise, these settings are fixed for all experiments.

\stitle{Evaluation Metrics}  
Following OpenBench~\cite{wang2025framework}, we evaluate profiling quality using Answer Accuracy (AC) and Answer Rate (AR). We report AC@T and AR@T under a turn budget $T$. AC@T is the fraction of questions answered correctly after a $T$-turn dialogue, and AR@T is the fraction for which the transcript can provide confident supporting references within $T$ turns. Together, these metrics quantify how well the agent elicits accurate and comprehensive information under limited interaction.
For reciprocal recommendation, we use Hit@K, where $K$ denotes the number of retrieved candidates per user. For a user $u_i$ with true match $u_j$, 
$Hit@K = \mathbb{I}\big(u_j \in \text{TopK}(u_i)\big)$.
\subsection{Profile Elicitation Performance (RQ1)}

Table~\ref{tab:compare_full} reports the performance of \ours and all baselines in terms of AC@T and AR@T on the OpenMatch dataset, across both romantic matchmaking and job interview scenarios. The best results are highlighted in bold, while the second-best results are underlined. Relative improvements over the best-performing baseline for each metric.
As shown in Table~\ref{tab:compare_full}, \ours consistently achieves the best results across all scenarios, dataset subsets, and evaluation metrics.
On the ROMANTIC-MAN subset, it improves AC@T and AR@T by an average of 10.9\% and 11.2\% across T$\in\{3,4,5\}$, with the largest margins of 13.6\% and 14.0\%.
On the ROMANTIC-WOMAN subset, \ours achieves comparable gains, improving the two metrics by an average of 9.7\% and 9.1\%, with the largest improvements of 10.4\% and 10.7\%.
On the JOB scenario, \ours improves two metrics by an average of 6.3\% and 5.8\%, with the largest margins of 11.2\% and 11.0\%.
These consistent improvements demonstrate that our adaptive profiling process can effectively probe more useful information from the respondent than the baselines.

Among the first-group baselines, \ie the interviewer agents from OpenBench with different LLM backbones, GPT-5 achieves the strongest performance in most settings on ROMANTIC-MAN and ROMANTIC-WOMAN, while DeepSeek-V3.1 is generally the best open-sourced backbone and delivers the second strongest results within the group, outperforming Gemini-2.5-Pro and other LLMs. In the JOB scenario, Llama-3.3-70B stands out as the strongest backbone in this group and outperforms both commercial and open-sourced alternatives.
For the second group of user-eliciting methods based on LLM prompt engineering, GATE attains better results in more cases and serves as the second-best among all compared methods in the romantic matchmaking scenario.
In the third baseline group, \randask shows relatively worse performance, further highlighting that random exploration, without proactively and dynamically considering the dialogue state, cannot effectively elicit user information with wide coverage and high confidence. It serves as an ablation study that underscores the effectiveness of our adaptive question selection strategy in guiding the dialogue.
A clear gap between the baselines and \ours remains across all three dataset subsets and all metrics, suggesting that neither random sampling nor prompting alone, even with a state-of-the-art LLM backbone and thinking mode, can accurately model how historical answers should guide future questions. This confirms that explicitly modeling user state and the question–answer relevance in \ours is necessary.

\input{tabls/abl_llms}
\input{tabls/abl_encode}
\subsection{\revise{Model Choice Study (RQ2)}}
\stitle{LLM Backbones}
To assess the robustness of \ours across different language model backbones when analyzing questions and real user answers, we compare GPT-4.1, GPT-5, and Gemini-2.5-flash on the OpenMatch dataset, regarding AC@5 and AR@5. 
As shown in Table~\ref{tab:llm_backbones}, different backbones yield highly consistent performance, with a maximum observed difference of about 1.8\%.
This indicates that our framework is not sensitive to the choice of backbone model, and its effectiveness mainly stems from the proposed adaptive profiling mechanism rather than from specific backbone choices.

\stitle{Embedding Models}
We further examine the influence of different embedding models utilized during the relevance prediction on the performance of \ours. All embeddings are fixed to 1024 dimensions for fair comparison. The evaluated models include text-embedding-3-small (O3-Small), text-embedding-3-large (O3-Large) and gemini-embedding-001 (G-text-001).
Among them, O3-Large and G-text-001 yield better performance, while O3-Small performs relatively worse.
The overall performance gap is bounded within {4.0\%}, suggesting that our method maintains robustness under different embedding choices and does not rely on a particular encoder.

\subsection{Reciprocal Matching Performance (RQ3)}

\input{tabls/match_compare}
In this subsection, we evaluate the end-to-end effectiveness of \ours in reciprocal matching on the OpenMatch dataset.
Each user participates in three independent 5-turn conversations for both self-information ($s$) and mate-preference ($m$).
We retrieve the top $K = \{1, 2, 5, 10\}$ matches from a pool of 100 candidates, and report the $Hit@K$ averaged across these runs to ensure the robustness of the results.
For the Non-transformed (Non-T) setting, we directly encode each user’s raw dialogue regarding self-information and mate-preference into embeddings, while the {Transformed} method applies the structured transformation introduced in Section~\ref{sec:user-rep} to summarize high-level insights and filter out superficial noises.

As shown in Table~\ref{tab:hit_results_mm},
the transformation-based representation consistently outperforms the non-transformed version across all candidate sizes. On the MAN-50 subset, it yields an average improvement of 59.2\% and up to 174.5\%. On the WOMAN-50 subset, it achieves an average gain of 67.8\% and up to 88.8\%. These consistent improvements demonstrate that the transformation step effectively enhances representation quality, leading to more accurate and robust reciprocal matching across different user groups.

\input{figs/quick_both}
\subsection{Efficient Matching Evaluation (RQ4)}

To assess the scalability and efficiency of \ourmatch, we run experiments on RandUser-100K using FAISS as the backend for range search and evaluate three index strategies: {IVF1024-Flat}, {IVF1024-SQ8}, and {HNSW32}.
To ensure a sufficient coverage of candidates, we set $nprobe=256$ for IVF-based indexes and $efSearch=256$ for HNSW. We randomly sample 100 query users to measure both recall and latency, repeat each experiment three times, and report the average results.  
We evaluate under six different matching $\theta$ thresholds ranging from 0.6 to 0.775. In this dataset, roughly 10.9\% of the user pairs have matching score exceed $\theta=0.60$, 3.7\% exceed $0.65$,
0.14\% exceed $0.75$, 
and 0.05\% exceed $0.775$. 
$\theta=0.775$ corresponds to a realistic high-precision setting, where on average it retrieves about 50 top matches per user out of the 100K candidates.

\stitle{Accuracy}  
Fig.~\ref{fig:recall-latency-vs-theta} reports the recall of different index types under varying thresholds. The recall generally increases with higher thresholds, as fewer high-score pairs need to be retrieved. For the practical matching cases with $\theta\ge0.75$, HNSW32 achieves near-exact recall of 0.958 and 0.989, respectively, showing that it can closely reproduce the exact results while being much faster. The IVF1024 indexes also maintain strong performance, with {IVF1024-Flat} achieving a recall of 0.7426 even when retrieving roughly 3–4 thousand candidates per query.  

\stitle{Efficiency}  
Fig.~\ref{fig:recall-latency-vs-theta} summarizes the average query latency under the same thresholds. The exact brute-force baseline costs about 2016.3\,ms per query. In contrast, {IVF1024-Flat} achieves an average 16.3$\times$ speedup, ranging from 8.5$\times$ to 20.8$\times$. {IVF1024-SQ8} further improves efficiency through quantization, reaching an average 16.8$\times$ speedup with a minimum of 9.0$\times$. The {HNSW32} index provides the best balance between accuracy and efficiency, with an average 18.4$\times$ speedup, up to 22.9$\times$ in high-threshold scenarios. All indexing methods exhibit the same trend: as $\theta$ increases, \ourmatch accelerates further because fewer candidates are explored.  

Overall, the experiments on RandUser-100K demonstrate that \ourmatch achieves both high effectiveness and scalability. Even under strict matching thresholds, the proposed method preserves near-exact recall while reducing latency significantly, confirming its practicality for large-scale reciprocal matching systems.

\subsection{\revise{Component Analysis (RQ5)}}\label{sec:rq5-components}
\stitle{Validation of the Relevance Mapping}
\revise{We validate the polynomial calibration curve that maps the answer-analysis similarity to the true relevance. We randomly split the unique 4,000 sampled question pairs into five folds, hold out one fold for testing each time, fit the cubic polynomial on the remaining four folds, and evaluate how well its predictions agree with the LLM-generated relevance ratings on the held-out fold. We report the Pearson and Spearman correlations, the root-mean-square error (RMSE), and the expected calibration error (ECE), averaged over all independent runs. As shown in Table~\ref{tab:relcal}, the learned polynomial mapping is well aligned with the LLM-generated ratings and substantially reduces the calibration error, lowering the ECE from $0.400$ (raw cosine) to $0.005$ while also improving the correlation and RMSE. This confirms that the calibrated relevance mapping is accurate and well calibrated.}

\begin{table}[t]
\centering
\revisecolor
\setlength{\tabcolsep}{4pt}
\caption{\revise{Validation of the relevance mapping.}}
\vspace{-3mm}
\begin{tabular}{lcccc}
\toprule
\textbf{Method} & \textbf{Pearson}\,$\uparrow$ & \textbf{Spearman}\,$\uparrow$ & \textbf{RMSE}\,$\downarrow$ & \textbf{ECE}\,$\downarrow$ \\
\midrule
Raw cosine        & 0.743 & 0.744 & 0.408 & 0.400 \\
Polynomial (deg=3) & \textbf{0.771} & 0.744 & \textbf{0.075} & \textbf{0.005} \\
\bottomrule
\end{tabular}
\label{tab:relcal}
\vspace{-2mm}
\end{table}

\stitle{Component Ablation}
\revise{To quantify the contribution of each core component, we conduct four finer-grained ablations on the ROM-MAN subset: (i) \emph{w/o Relevance} removes the calibrated relevance predictor; (ii) \emph{w/o Redundancy} removes the redundancy-penalty term in the user-state update and keeps only the marginal state gain; (iii) \emph{+ Diversity} replaces the information-gain selection with a diversity-driven rule that selects the question least related to those already asked; and (iv) \emph{+ Q*} replaces the \oursub-generated question set with $100$ questions directly generated by the LLM. As shown in Table~\ref{tab:abl_nbq}, \ours consistently achieves the best performance across all metrics. The largest gains are $+6.0\%$ AC@5 / $+4.4\%$ AR@5 over \emph{w/o Relevance}, $+3.0\%$ AC@3 / $+2.5\%$ AR@3 over \emph{w/o Redundancy}, $+6.7\%$ AC@4 / $+6.3\%$ AR@3 over \emph{+ Diversity}, and $+5.9\%$ AC@4 / $+4.5\%$ AR@5 over \emph{+ Q*}.
These results indicate that each core component, \ie the calibrated relevance, the redundancy handling, the information-gain-based selection, and the comprehensive question-set generation, contributes to the overall profiling quality.}

\begin{table}[t]
\centering
\revisecolor
\setlength{\tabcolsep}{2pt}
\caption{\revise{\ours component ablation study on ROM-MAN.}}
\vspace{-3mm}
\begin{adjustbox}{max width=\columnwidth}
\begin{tabular}{lcccccc}
\toprule
\textbf{Variant} & AC@3 & AR@3 & AC@4 & AR@4 & AC@5 & AR@5 \\
\midrule
\ours w/o Relevance  & 0.691 & 0.764 & 0.777 & 0.839 & 0.797 & 0.857 \\
\ours w/o Redundancy & 0.682 & 0.748 & 0.769 & 0.823 & 0.840 & 0.885 \\
\ours + Diversity    & 0.660 & 0.722 & 0.743 & 0.810 & 0.800 & 0.850 \\
\ours + Q*           & 0.679 & 0.746 & 0.749 & 0.810 & 0.800 & 0.857 \\
\midrule
\ours                & \textbf{0.703} & \textbf{0.767} & \textbf{0.793} & \textbf{0.843} & \textbf{0.845} & \textbf{0.895} \\
\bottomrule
\end{tabular}
\end{adjustbox}
\label{tab:abl_nbq}
\vspace{-2mm}
\end{table}

%% file: tabls/main_compare.tex
\begin{table*}[!t]
\centering
\large
\setlength{\tabcolsep}{2pt}
\caption{Performance evaluation in AC@T and AR@T on OpenMatch across three scenarios.}\vspace{-3mm}
\begin{adjustbox}{max width=\textwidth}
\begin{tabular}{lcccccccccccccccccc}
\toprule
\multirow{2}{*}{\textbf{Models}}
& \multicolumn{6}{c}{\man}
& \multicolumn{6}{c}{\woman}
& \multicolumn{6}{c}{\textbf{JOB}} \\
\cmidrule(lr){2-7} \cmidrule(lr){8-13} \cmidrule(lr){14-19}
& AC@3 & AR@3 & AC@4 & AR@4 & AC@5 & AR@5
& AC@3 & AR@3 & AC@4 & AR@4 & AC@5 & AR@5
& AC@3 & AR@3 & AC@4 & AR@4 & AC@5 & AR@5 \\
\midrule
GPT-5        & \underline{0.619} & \underline{0.673} & 0.702 & 0.752 & 0.778 & 0.822
             & 0.600 & 0.649 & 0.700 & \underline{0.746} & 0.777 & 0.811
             & 0.595 & 0.613 & 0.703 & 0.722 & 0.743 & 0.748 \\
GPT-5-mini   & 0.606 & 0.662 & 0.677 & 0.737 & 0.772 & 0.817
             & 0.585 & 0.632 & 0.661 & 0.705 & 0.748 & 0.770
             & 0.605 & 0.620 & 0.683 & 0.700 & 0.742 & 0.758 \\
Gemini-2.5-Pro & 0.602 & 0.665 & 0.687 & 0.741 & 0.764 & 0.809
             & 0.597 & 0.640 & 0.679 & 0.722 & 0.769 & 0.800
             & 0.580 & 0.598 & 0.665 & 0.687 & 0.720 & 0.735 \\
Qwen-3-80B   & 0.537 & 0.597 & 0.616 & 0.669 & 0.679 & 0.733
             & 0.533 & 0.563 & 0.608 & 0.651 & 0.656 & 0.689
             & 0.593 & 0.608 & 0.682 & 0.695 & 0.753 & 0.763 \\
Deepseek-V3.1 & 0.608 & 0.667 & 0.701 & 0.757 & \underline{0.783} & \underline{0.825}
             & 0.583 & 0.622 & 0.697 & 0.733 & 0.763 & 0.783
             & 0.605 & 0.627 & 0.705 & 0.715 & 0.783 & 0.788 \\
Llama-3.3-70B & 0.613 & 0.660 & 0.689 & 0.737 & 0.757 & 0.808
             & 0.581 & 0.623 & 0.686 & 0.707 & 0.760 & 0.790
             & 0.616 & 0.627 & 0.723 & 0.736 & \underline{0.785} & \underline{0.797} \\
\midrule
MockLLM      & 0.549 & 0.611 & 0.654 & 0.709 & 0.739 & 0.779
             & 0.538 & 0.588 & 0.630 & 0.668 & 0.727 & 0.762
             & 0.587 & 0.600 & 0.627 & 0.642 & 0.665 & 0.678 \\
GATE         & 0.604 & 0.655 & \underline{0.714} & \underline{0.758} & 0.779 & 0.817
             & \underline{0.624} & 0.663 & \underline{0.705} & 0.745 & \underline{0.782} & \underline{0.818}
             & 0.612 & 0.632 & 0.658 & 0.673 & 0.708 & 0.718 \\
\midrule
\randask     & 0.595 & 0.651 & 0.677 & 0.729 & 0.735 & 0.784
             & 0.619 & \underline{0.672} & 0.687 & 0.739 & 0.738 & 0.775
             & \underline{0.698} & \underline{0.713} & \underline{0.760} & \underline{0.770} & 0.778 & 0.792 \\
\midrule
\ours        & \makecell{\textbf{0.703}\\ \normalsize +13.6\%} &
               \makecell{\textbf{0.767}\\ \normalsize +14.0\%} &
               \makecell{\textbf{0.793}\\ \normalsize +11.1\%} &
               \makecell{\textbf{0.843}\\ \normalsize +11.2\%} &
               \makecell{\textbf{0.845}\\ \normalsize +7.9\%}  &
               \makecell{\textbf{0.895}\\ \normalsize +8.5\%}  &
               \makecell{\textbf{0.689}\\ \normalsize +10.4\%} &
               \makecell{\textbf{0.744}\\ \normalsize +10.7\%} &
               \makecell{\textbf{0.775}\\ \normalsize +9.9\%}  &
               \makecell{\textbf{0.817}\\ \normalsize +9.5\%}  &
               \makecell{\textbf{0.850}\\ \normalsize +8.7\%}  &
               \makecell{\textbf{0.877}\\ \normalsize +7.2\%}  &
               \makecell{\textbf{0.707}\\ \normalsize +1.3\%}  &
               \makecell{\textbf{0.713}\\ \normalsize +0.0\%}  &
               \makecell{\textbf{0.808}\\ \normalsize +6.3\%}  &
               \makecell{\textbf{0.820}\\ \normalsize +6.5\%}  &
               \makecell{\textbf{0.873}\\ \normalsize +11.2\%} &
               \makecell{\textbf{0.885}\\ \normalsize +11.0\%} \\
\bottomrule
\end{tabular}
\end{adjustbox}
\vspace{-2mm}
\label{tab:compare_full}
\end{table*}

%% file: tabls/abl_llms.tex

\begin{table}[t]
\centering 
\setlength{\tabcolsep}{2pt}
\caption{\ours with different LLM backbones.}
\vspace{-3mm}
\begin{tabular}{lcccccc}
\toprule
\textbf{Backbones} & \multicolumn{2}{c}{\shortman} & \multicolumn{2}{c}{\shortwoman} & \multicolumn{2}{c}{\textbf{JOB}} \\
\cmidrule(lr){2-3} \cmidrule(lr){4-5} \cmidrule(lr){6-7}
& AC@5 & AR@5 & AC@5 & AR@5 & AC@5 & AR@5 \\
\midrule
GPT-4.1          & {\underline{0.845}} & {\underline{0.895}} & {0.850} & 0.877 & 0.873 & 0.885 \\
GPT-5            & \underline{0.845} & 0.885 & \underline{0.857} & {\underline{0.884}} & \textbf{0.880} & \textbf{0.888} \\
Gemini-2.5-flash & {\underline{0.845}} & \textbf{{0.897}} & \textbf{{0.865}} & \textbf{{0.888}} & {\underline{0.875}} & 0.885 \\
\bottomrule
\end{tabular}
\vspace{-2mm}
\label{tab:llm_backbones}
\end{table}

%% file: tabls/abl_encode.tex

\begin{table}[t]
\centering 
\setlength{\tabcolsep}{2pt}
\caption{\ours with different encoding models.}
\vspace{-3mm}
\begin{tabular}{lcccccc}
\toprule
\textbf{Encoding} & \multicolumn{2}{c}{\shortman} & \multicolumn{2}{c}{\shortwoman} & \multicolumn{2}{c}{\textbf{JOB}} \\
\cmidrule(lr){2-3} \cmidrule(lr){4-5} \cmidrule(lr){6-7}
& AC@5 & AR@5 & AC@5 & AR@5 & AC@5 & AR@5 \\
\midrule
O-3-large  & {\underline{0.845}} & \textbf{{0.895}} & \textbf{{0.850}} & \textbf{{0.877}} & 0.873 & 0.885 \\
O-3-small  & 0.829 & 0.887 & 0.817 & 0.844 & {\underline{0.883}} & {\underline{0.897}} \\
G-text-001 & \textbf{{0.849}} & {\underline{0.893}} & {\underline{0.831}} & {\underline{0.860}} & \textbf{{0.897}} & \textbf{{0.903}} \\
\bottomrule
\end{tabular}\vspace{-2mm}
\label{tab:encoding_models}
\end{table}

%% file: tabls/match_compare.tex
\begin{table}[t]
\centering
\caption{Matching accuracy with different representations.}
\vspace{-3mm}
\begin{tabular}{lcccc}
\toprule
\multirow{2}{*}{\textbf{Metric}}
& \multicolumn{2}{c}{\shortman}
& \multicolumn{2}{c}{\shortwoman} \\
\cmidrule(lr){2-3}\cmidrule(lr){4-5}
& Non-T & Transformed & Non-T & Transformed \\
\midrule
Hit@1  & 0.027 & \textbf{0.073} {\small (+174.5\%)} & 0.040 & \textbf{0.060} {\small (+50.0\%)} \\
Hit@2  & 0.093 & 0.093 {\small (+0.0\%)}          & 0.060 & \textbf{0.113} {\small (+88.8\%)} \\
Hit@5  & 0.180 & \textbf{0.213} {\small (+18.5\%)} & 0.140 & \textbf{0.253} {\small (+80.9\%)} \\
Hit@10 & 0.287 & \textbf{0.413} {\small (+44.2\%)} & 0.247 & \textbf{0.373} {\small (+51.3\%)} \\
\bottomrule
\end{tabular}
\label{tab:hit_results_mm}
\vspace{-2mm}
\end{table}

%% file: figs/quick_both.tex

\begin{figure}[t]
\centering
\hspace{-3mm}
\begin{tikzpicture}

\begin{groupplot}[
  group style={
    group size=2 by 1,
    horizontal sep=12mm, 
  },
  width=0.56\columnwidth,
  height=0.56\columnwidth,
  xmin=0.59, xmax=0.785,
  xtick={0.60,0.65,0.70,0.725,0.75,0.775},
  xticklabel={
    \pgfmathparse{\tick}%
    \ifdim\pgfmathresult pt=0pt
      0%
    \else
      \pgfmathprintnumber[
        fixed,
        precision=3,
        skip 0.
      ]{\tick}%
    \fi
  },
  xticklabel style={
    rotate=90, anchor=east,
    font=\small
  },
  yticklabel style={font=\small},
  label style={font=\small},
  enlarge x limits=0.02,
  every axis plot/.style={black, line width=1.0pt},
  xlabel={}, 
]

\nextgroupplot[
  ymin=0.22, ymax=1.05,
  ylabel={Recall},
  ytick={0.2,0.4,0.6,0.8,1.0},
  legend to name=globallegend,
  legend style={
    font=\small,
    draw=none,
    fill=white,
    fill opacity=0.9,
    /tikz/every even column/.append style={column sep=10pt},
  },
  legend columns=3,
]

\addplot+[color=red, mark=diamond, mark size=3pt,
          mark options={fill=white, line width=.5pt, solid}, solid, line width=.5pt]
  coordinates {(0.60,0.6284) (0.65,0.7426) (0.70,0.8453) (0.725,0.8914) (0.75,0.9287) (0.775,0.9614)};
\addlegendentry{IVF1024-Flat}

\addplot+[color=blue, mark=x, mark size=2pt,
          mark options={line width=.5pt, solid}, dotted, line width=1.0pt]
  coordinates {(0.60,0.6281) (0.65,0.7420) (0.70,0.8443) (0.725,0.8898) (0.75,0.9259) (0.775,0.9582)};
\addlegendentry{IVF1024-SQ8}

\addplot+[color=black, mark=o, mark size=2pt,
          mark options={line width=.5pt, solid}, solid, line width=.5pt]
  coordinates {(0.60,0.3062) (0.65,0.5324) (0.70,0.7921) (0.725,0.8912) (0.75,0.9576) (0.775,0.9891)};
\addlegendentry{HNSW32}

\nextgroupplot[
  ymin=75, ymax=250,
  ylabel={Latency (ms)},
  ytick={100,150,200,250},
]

\addplot+[color=red, mark=diamond, mark size=3pt,
          mark options={fill=white, line width=.5pt, solid}, solid, line width=.5pt]
  coordinates {(0.60,238.4) (0.65,166.6) (0.70,111.0) (0.725,103.9) (0.75,109.0) (0.775,97.13)};

\addplot+[color=blue, mark=x, mark size=2pt,
          mark options={line width=.5pt, solid}, dotted, line width=1.0pt]
  coordinates {(0.60,223.8) (0.65,148.0) (0.70,119.1) (0.725,100.3) (0.75,102.4) (0.775,97.43)};

\addplot+[color=black, mark=o, mark size=2pt,
          mark options={line width=.5pt, solid}, solid, line width=.5pt]
  coordinates {(0.60,155.3) (0.65,138.4) (0.70,111.5) (0.725,98.30) (0.75,95.35) (0.775,88.18)};

\end{groupplot}

\node at ($(group c1r1.north)!0.5!(group c2r1.north)$) [above=.5mm]
  {\pgfplotslegendfromname{globallegend}};

\node[font=\small] at ($(group c1r1.south)!0.5!(group c2r1.south)$) [below=6mm] {$\theta$};

\end{tikzpicture}
\vspace{-3mm}
\caption{Recall and latency under varying $\theta$.}
\vspace{-3mm}
\label{fig:recall-latency-vs-theta}
\end{figure}
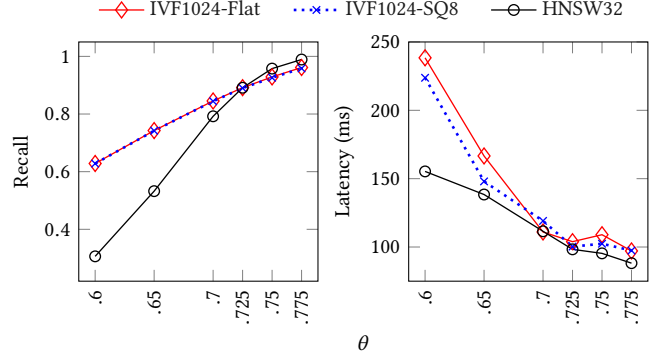

%% file: sections/related.tex
\section{Related Work}

\stitle{AI Interviewer \& Conversation Planner}
Our work is closest to agents that actively plan multi-turn interviews to profile users via question selection or generation. Recent methods learn dialogue policies to optimize long-horizon conversation rewards, including (hierarchical) policy-learning frameworks in Conversational Recommender Systems (CRS)~\cite{www2023hicr,lin2023personalization}, retrieval-augmented pipelines that ground turns on external evidence~\cite{kemper2024ragconvrec}, and planning approaches that model conversations as stochastic processes in a semantic space~\cite{chen2025scope}.
However, a key challenge is that informativeness rewards are hard to define because supervision often comes from downstream outcomes (e.g., customer purchases), and most methods rely on task-specific training or external knowledge; empirical CRS studies further highlight the difficulty to align generic conversation rewards with information-gain–-centric profiling~\cite{zhang2024empirical}.
In contrast, our policy is plug-and-play and selects the next best question without retraining. 
Recent related elicitation systems such as MockLLM~\cite{sun2025mockllm} and GATE~\cite{lieliciting} use prompting to guide question asking (e.g., avoiding repetition and maximizing marginal information gain), but they do not explicitly model question--answer relevance and do not assume a strict turn budget.
\revise{Active learning also studies informativeness-driven querying~\cite{settles2009active}, but its standard setting selects unlabeled samples for annotation. \ours addresses a different problem: choosing open-ended natural-language questions in an evolving dialogue and estimating their value from partial, free-form answers under a tight turn budget.}

\stitle{Reciprocal Recommender System}
Reciprocal recommenders optimize mutual satisfaction, with online dating as a canonical domain. Classical systems (e.g., RECON~\cite{pizzato2010recon}) and later large-scale design studies optimize for mutual interest using static attributes or historical interactions~\cite{xia2016design}. Recent work incorporates richer two-sided signals, such as reciprocal sequential recommendation (ReSeq)~\cite{zheng2023reseq} and knowledge-aware, explainable matching (KAERR)~\cite{lai2024knowledge}, typically operating on existing structured profiles, logged interactions, or knowledge graphs. 
LLM-based methods can further enrich existing profiles before matching; for instance, MockLLM~\cite{sun2025mockllm} synthesizes pairwise interview dialogues and uses LLMs for two-sided evaluation, but this is mainly a re-ranking strategy over small candidate sets because exhaustive pairwise LLM evaluation does not scale to large platforms. 
Unlike these works that assume profiles already exist, we study reciprocal matching in a cold-start, interview-style setting: \ours elicits user profiles via a conversation and make real-time matches directly based on the resulting dialogues.

We defer approximate vector search related work to Appendix~\ref{app:related}.


%% file: sections/conclusion.tex
\section{Conclusion}
In this work, we formalize the NBQ problem and its challenging downstream application, reciprocal matchmaking.
To address the inherent randomness in users’ open-ended responses, \ours adaptively selects the next best question from a pre-generated comprehensive question set to greedily maximize the incremental information gain at each conversation turn.
We also propose a transformation-based user representation that extracts high-level insights from user dialogues and supports arbitrary downstream tasks.
Finally, \ourmatch replaces the quadratic scan with approximate range-search--based pre-filtering, enabling low-latency retrieval of highly compatible matches at scale.
\revise{One limitation is that adapting \ours to a new scenario still requires some manual initialization. Automating this scenario-level setup is an important direction for future work.}


%% file: sections/appendix.tex
\appendix
\input{sections/proof}

\section{Algorithm Pseudocode}\label{app:alg}
Here we present the pseudocode of the question set generation algorithm \texttt{SubdimDiv}, described in Section~\ref{sec:qgen} and shown in Algorithm~\ref{alg:subdimdiv}.
\begin{algorithm}[t]
\caption{\texttt{SubdimDiv}}\label{alg:subdimdiv}
\KwIn{Application scenario $\alpha$, root dimension set $\mathcal{D}_{0}$, minimum and maximum divisions $(k_{\min}, k_{\max})$, and tree height $L$}
\KwOut{Sub-dimension set $\mathcal{D}$}
$\mathcal{D}_0 \gets \mathcal{D}_{0}$\ ;\\ $\mathcal{D}_l \gets \emptyset,\ \forall l \in \{1,\ldots,L\}$;\\
\For{$t \gets 1$ \textbf{to} $L$}{
    \For{each $d \in \mathcal{D}_{t-1}$}{
        $S \gets \texttt{LLMDivide}(d, \alpha, k_{\min}, k_{\max})$;\\
        $\mathcal{D}_t \gets \mathcal{D}_t \cup S$;\\
    }
}
\Return {$\mathcal{D} = \{\mathcal{D}_0, \ldots, \mathcal{D}_L\}$;}
\end{algorithm}

\input{sections/synthesis}

\section{Additional Related Work}\label{app:related}
\stitle{Approximate Vector Search}
There has been a long line of work on approximate vector search in high-dimensional vector space. These well-studied fields, including approximate nearest neighbor search (ANNS) and maximum inner product search (MIPS), differ in metric or non-metric settings, but share the goal of balancing query efficiency and accuracy when retrieving the most \textit{similar} candidates for the query vector. Their methodologies can be briefly classified into three families: hashing~\cite{gionis1999similarity, tao2009quality, shrivastava2014asymmetric, neyshabur2015symmetric, tian2022db, zhao2024efficient}, quantization~\cite{jegou2010product, gong2012iterative}, and proximity graphs~\cite{malkov2018efficient, zhao2020song, ootomo2024cagra, song2024efficient, luo2025tag}.
Unlike most of these methods that target \textit{top-k} queries, a range query retrieves all vectors within a given distance (or \textit{radius}) of the query vector. In database terms, this problem is known as a \textit{similarity join}, which outputs all pairs of vectors with similarity higher than a given threshold. Range search queries can be supported by indexing techniques for ANNS and MIPS. Recent efforts, including SimJoin~\cite{xie2025fast}, go beyond independent per-point query searches by leveraging relationships between partial join results.
In the reciprocal matching scenario, it is common to retrieve all candidate pairs with good compatibility, \ie higher than a pre-defined matching score threshold, which shares a similar setting with vector range search. However, it differs in that the mutual compatibility is based on the product of embedding similarities in both directions and thus cannot be directly supported by the approximate vector search techniques mentioned above. \ourmatch serves as a novel approach to reduce this gap by introducing a practically tight relaxation, augmenting the vector-based representations, and quantitatively bridging the threshold values on both sides.

\input{sections/prompts}

%% file: sections/proof.tex
\section{Proof of Theorems}\label{app:a}
\subsection{Proof of Theorem~\ref{thm:1}}
\setcounter{theorem}{0}
\begin{theorem}
Let $f_{ij},f_{ji}\in(0,1]$. If $g(u_i,u_j) \geq \theta$, then necessarily $[\bar{\mathbf{s}}_i;\,\bar{\mathbf{m}}_i] \cdot [\bar{\mathbf{m}}_j;\,\bar{\mathbf{s}}_j]
\;\ge\; 4\theta - 2$ where $\bar{\mathbf{m}} = \frac{\mathbf{m}}{\|\mathbf{m}\|}$ and $\bar{\mathbf{s}} = \frac{\mathbf{s}}{\|\mathbf{s}\|}.$
\end{theorem}

\begin{proof}
Assume $g(u_i,u_j)\ge \theta$. By the definition of $g$ (Eq.~\eqref{eq:g_def}),
$$
g(u_i,u_j)=\frac{2 f_{ij}\cdot f_{ji}}{f_{ij}+f_{ji}}.
$$
We first show that
\begin{equation}
\frac{2 f_{ij}\cdot f_{ji}}{f_{ij}+f_{ji}} \le \frac{f_{ij}+f_{ji}}{2}.
\label{eq:harm_le_arith}
\end{equation}
Indeed,
\[
\frac{f_{ij}+f_{ji}}{2} - \frac{2 f_{ij}\cdot f_{ji}}{f_{ij}+f_{ji}}
= \frac{(f_{ij}-f_{ji})^2}{2(f_{ij}+f_{ji})} \;\ge\; 0,
\]
where the denominator is nonnegative since $f_{ij},f_{ji}\in[0,1]$ and $f_{ij}+f_{ji}>0$.
Combining $g(u_i,u_j)\ge \theta$ with \eqref{eq:harm_le_arith} yields
\[
\theta \le g(u_i,u_j) \le \frac{f_{ij}+f_{ji}}{2},
\]
and thus
\begin{equation}
f_{ij}+f_{ji} \ge 2\theta.
\label{eq:sum_lb}
\end{equation}

Next, expand $f_{ij}$ and $f_{ji}$ in terms of representation embeddings:
\[
f_{ij}=\frac{1+\bar{\mathbf{m}}_i\cdot \bar{\mathbf{s}}_j}{2},
\qquad
f_{ji}=\frac{1+\bar{\mathbf{s}}_i\cdot \bar{\mathbf{m}}_j}{2}.
\]
Where $\bar{\mathbf{m}} = \frac{\mathbf{m}}{\|\mathbf{m}\|}$ and $\bar{\mathbf{s}} = \frac{\mathbf{s}}{\|\mathbf{s}\|}$.
Therefore,
\begin{align*}
f_{ij}+f_{ji}
&= 1 + \frac{1}{2}\left(\bar{\mathbf{m}}_i\cdot \bar{\mathbf{s}}_j
+ \bar{\mathbf{s}}_i\cdot \bar{\mathbf{m}}_j\right) \\
&= 1 + \frac{1}{2}\,[\bar{\mathbf{s}}_i;\bar{\mathbf{m}}_i]\cdot[\bar{\mathbf{m}}_j;\bar{\mathbf{s}}_j].
\end{align*}
Plugging this into \eqref{eq:sum_lb} gives
\[
1 + \frac{1}{2}\,[\bar{\mathbf{s}}_i;\bar{\mathbf{m}}_i]\cdot[\bar{\mathbf{m}}_j;\bar{\mathbf{s}}_j]
\;\ge\; 2\theta,
\]
which is equivalent to
\[
[\bar{\mathbf{s}}_i;\,\bar{\mathbf{m}}_i] \cdot [\bar{\mathbf{m}}_j;\,\bar{\mathbf{s}}_j]
\;\ge\; 4\theta - 2.
\]
This completes the proof.
\end{proof}

\subsection{Proof of Theorem~\ref{thm:2}}
\setcounter{theorem}{1}
\begin{theorem}\label{thm:2}
Let $f_{ij},f_{ji}\in(0,1]$ and $\theta\geq\frac{1}{2}$. If a user pair $(u_i, u_j)$ satisfies $[\bar{\mathbf{s}}_i;\,\bar{\mathbf{m}}_i] \cdot [\bar{\mathbf{m}}_j;\,\bar{\mathbf{s}}_j]
\;\ge\; 4\theta - 2$ then $g(u_i, u_j) \geq \theta - \frac{(1-\theta)^2}{\theta}.$
\end{theorem}

\begin{proof}
Let $s=f_{ij}+f_{ji}$. By definition of $g(u_i,u_j)$, we have:
\begin{equation}\label{eq:am-hm-gap}
\begin{aligned}
\frac{f_{ij}+f_{ji}}{2}-g(u_i,u_j)
&=\frac{f_{ij}+f_{ji}}{2}-\frac{2f_{ij}\cdot f_{ji}}{f_{ij}+f_{ji}}\\
&=\frac{(f_{ij}-f_{ji})^2}{2(f_{ij}+f_{ji})}
=\frac{(f_{ij}-f_{ji})^2}{2s}.
\end{aligned}
\end{equation}
Thus, for a fixed $s$, the gap in Eq.~\eqref{eq:am-hm-gap} is monotonically increasing with $|f_{ij}-f_{ji}|$.

Given $f_{ij},f_{ji}\in (0,1]$ and $f_{ij}+f_{ji}=s$, the maximum of $|f_{ij}-f_{ji}|$ is achieved when one term is as large as possible and the other as small as possible.
Specifically, if $s\le 1$, the maximum is attained at $(f_{ij},f_{ji})=(s,0)$, yielding $|f_{ij}-f_{ji}|=s$.
If $s\ge 1$, the maximum is attained at $(f_{ij},f_{ji})=(1,s-1)$, yielding $|f_{ij}-f_{ji}|=2-s$.
Therefore, we have $|f_{ij}-f_{ji}|\le \min\{s,\,2-s\}$, and hence
\begin{equation}\label{eq:max-gap-square}
(f_{ij}-f_{ji})^2 \le \bigl(\min\{s,\,2-s\}\bigr)^2.
\end{equation}
Using $\min\{s,2-s\}=2\min\{s/2,\,1-s/2\}$, we rewrite Eq.~\eqref{eq:max-gap-square} as
\begin{equation*}\label{eq:max-gap-square-half}
(f_{ij}-f_{ji})^2 \le 4\bigl(\min\{s/2,\,1-s/2\}\bigr)^2.
\end{equation*}
Plugging this into Eq.~\eqref{eq:am-hm-gap} gives
\begin{equation}\label{eq:gap-bound-s}
\begin{aligned}
\frac{f_{ij}+f_{ji}}{2}-g(u_i,u_j)
&\le
\frac{4(\min\{s/2,1-s/2\})^2}{2s}\\
&=\frac{(\min\{s/2,\,1-s/2\})^2}{s/2}.    
\end{aligned}
\end{equation}

Now consider a pair $(u_i,u_j)$ satisfying
$[\bar{\mathbf{s}}_i;\bar{\mathbf{m}}_i]\cdot[\bar{\mathbf{m}}_j;\bar{\mathbf{s}}_j]\ge 4\theta-2$.
By Theorem~1, this implies $\frac{f_{ij}+f_{ji}}{2}\ge\theta$, i.e., $s\ge 2\theta$.
The RHS of Eq.~\eqref{eq:gap-bound-s} is maximized when $s$ is minimized, so the worst case occurs at $s=2\theta$.
Substituting $s=2\theta$ into Eq.~\eqref{eq:gap-bound-s} yields
\[
\frac{f_{ij}+f_{ji}}{2}-g(u_i,u_j)
\le
\frac{(\min\{\theta,1-\theta\})^2}{\theta}.
\]
Finally, when $\theta\ge \tfrac12$, we have $\min\{\theta,1-\theta\}=1-\theta$, and hence
\begin{equation*}
g(u_i,u_j) \ge \frac{f_{ij} + f_{ji}}{2} - \frac{(1-\theta)^2}{\theta} \ge \theta-\frac{(1-\theta)^2}{\theta},
\end{equation*}
which proves the theorem.
\end{proof}

%% file: sections/synthesis.tex
\section{Dataset Synthesis}\label{app:data}
Following the synthesis and evaluation pipeline of OpenBench~\cite{wang2025framework}, we synthesize the OpenMatch dataset under two scenarios: romantic matchmaking and job interviews.
We synthesize 100 respondents for the romantic matchmaking scenario and 20 for the job interview scenario.
The matchmaking respondents are divided into two subsets, ROMANTIC-MAN and ROMANTIC-WOMAN, each containing 50 respondents.
For both scenarios, we first identify 10 core dimensions as the knowledge discovery targets.
For each respondent, we then generate a \emph{high-level insight} for every dimension, and use these insights to construct an experience-based script with 10 topics.

\stitle{High-level insight synthesis}
For each respondent $u$, scenario, and script type ({self-information} and, when applicable, {match-preference}), we instruct LLMs to produce a concise, high-level \textit{insight} per dimension that captures the respondent's trait or stance.
For example, in romantic matchmaking, a \textit{Social Lifestyle} insight could be ``Quiet Companionability Seeker'';
in job interviews, a \textit{Work Identity} insight could be ``Fractional Head-of-Product for Seed-Stage SaaS.''
To encourage {inter-user diversity} at the insight level, we generate insights in a {batch mode}.
Specifically, we synthesize insights for a small batch of respondents at a time.
When generating the insight of a target respondent on a given dimension, we list the already-generated insights of other respondents in the same batch for that dimension, and explicitly require the model to propose a {different} insight for the target respondent.
This constraint reduces duplicates and increases the diversity within each scenario and script type.

\stitle{From insights to script topics via coherent personal events}
We next translate high-level insights into concrete, experience-based script topics used by OpenBench.
For each respondent script, we construct 10 topics.
Each topic is grounded in a single personal event or episode that naturally integrates \emph{two} randomly selected high-level insights of that respondent.
Concretely, we sample two insights (e.g., \textit{Social Lifestyle} and \textit{Communication Style}) and prompt the LLM to generate a coherent event narrative that jointly reflects both insights, rather than describing them in isolation.
To avoid \emph{intra-user inconsistency}, when generating each event we provide the respondent's remaining insights across other dimensions as context, and instruct the model to avoid conflicts across topics and dimensions (e.g., mutually exclusive preferences or incompatible identities).
These 10 event narratives compose the respondent's OpenBench {script}.

\stitle{Behavior patterns and response simulation}
Following OpenBench, we annotate each topic with an appropriate behavior pattern (e.g., realistic expansion in responses, referential consistency).
During answer simulation, the LLM plays the role of respondent $u$ conditioned on her script.
When asked a question that touches a particular dimension, the respondent is instructed to retrieve and describe a relevant personal event aligned with that dimension (or the most related topic), and then answer the question grounded in that event.
This design encourages answers to be natural, while remaining faithful to the respondent's underlying high-level insights.

\stitle{Scenario- and script-specific settings}
For romantic matchmaking, we synthesize two scripts per respondent: (i) a \textit{self-information} script describing who they are, and (ii) a \textit{match-preference} script describing what they seek in a partner.
For job interviews, we synthesize only the \textit{self-information} script.
For each script, we generate 10 topics and a corresponding questionnaire of 10 questions, resulting in 1{,}200 questions in total.

\stitle{Reciprocal matching pairs and preference synthesis}
For the reciprocal matching evaluation, we randomly construct 50 distinct one-to-one match pairs across the two respondent subsets for simplicity.
That is, if user $u_i$ is matched with user $u_j$, then $u_j$ is also matched with $u_i$.
Although the matches are randomly assigned, we treat them as pseudo-ground-truth for evaluation.
For each pair $(u_i, u_j)$, we generate $u_i$'s mate-preference script conditioned on $u_j$'s self-information script, and vice versa, ensuring that preferences are explicitly aligned with the partner's self profile.

%% file: sections/prompts.tex
\section{Prompt Review}\label{app:prompts}
\input{prompts/prompt_qgen}
\subsection{Question generation prompt}\label{app:prompts:1}
The question generation prompts are demonstrated in Fig.~\ref{prompt:qgen}.

\input{prompts/prompts_trans}
\subsection{Representation transformation prompt}\label{app:prompts:2}
For the user representation transformation, we instruct LLMs to conduct the following main operations:
\begin{itemize}
[topsep=2pt,itemsep=1pt,parsep=0pt,partopsep=0pt,leftmargin=11pt]

    \item \textbf{Insight Extraction}: Instead of simply rephrasing the Q\&A content, to extract deeper insights into the user’s intrinsic traits and relational needs. With cautious, well-founded inference on the respondent's natural characteristics and core preferences.
    
   \item \textbf{Key Signal Retention}: Critical signals for matching must be retained, including must-have and deal-breaker requirements.
    
    \item \textbf{Structured Expression}: The transformed paragraph should follow a uniform narrative structure with a clear beginning, detailed body, and brief conclusion, using direct descriptive language. Each paragraph should contain about 130--180 words to ensure clarity and comparability across users.
\end{itemize}
The full version of transformation prompts (for self-information dialogue) is shown in Fig.~\ref{prompt:trans}.

\input{prompts/prompts_rel}
\subsection{Relevance analysis prompt}\label{app:prompts:3}
The generative relevance analysis prompts are shown in Fig.~\ref{prompt:rel}.
To take the certainty-weighted average, we multiply the mean overlap rate of each level by the model’s certainty and then sum across all six levels.

\input{prompts/prompts_analysis}
\subsection{Question analysis prompt}\label{app:prompts:4}
When analyzing the answer elements for a given question, we have the following key requirements:
\begin{itemize}
[topsep=2pt,itemsep=1pt,parsep=0pt,partopsep=0pt,leftmargin=11pt]
    \item \textbf{Non-overlapping}: Each element should capture a unique aspect of the answer, avoiding redundancy or overlap.
    
    \item \textbf{Balanced}: All answer elements should be similar in granularity and scope, ensuring no single point dominates the others.
    
    \item \textbf{Complete}: The set of answer elements together should comprehensively cover the answer space of the question.
    
    \item \textbf{Self-contained}: Each answer element must be independently understandable. Abstract pronouns and vague terms are explicitly avoided. When needed, the phrasing may reiterate critical contextual elements to preserve specificity.
\end{itemize}
The full version of question analysis prompts is shown in Fig.~\ref{prompt:analysis}.

%% file: prompts/prompt_qgen.tex
\begin{figure*}[h]
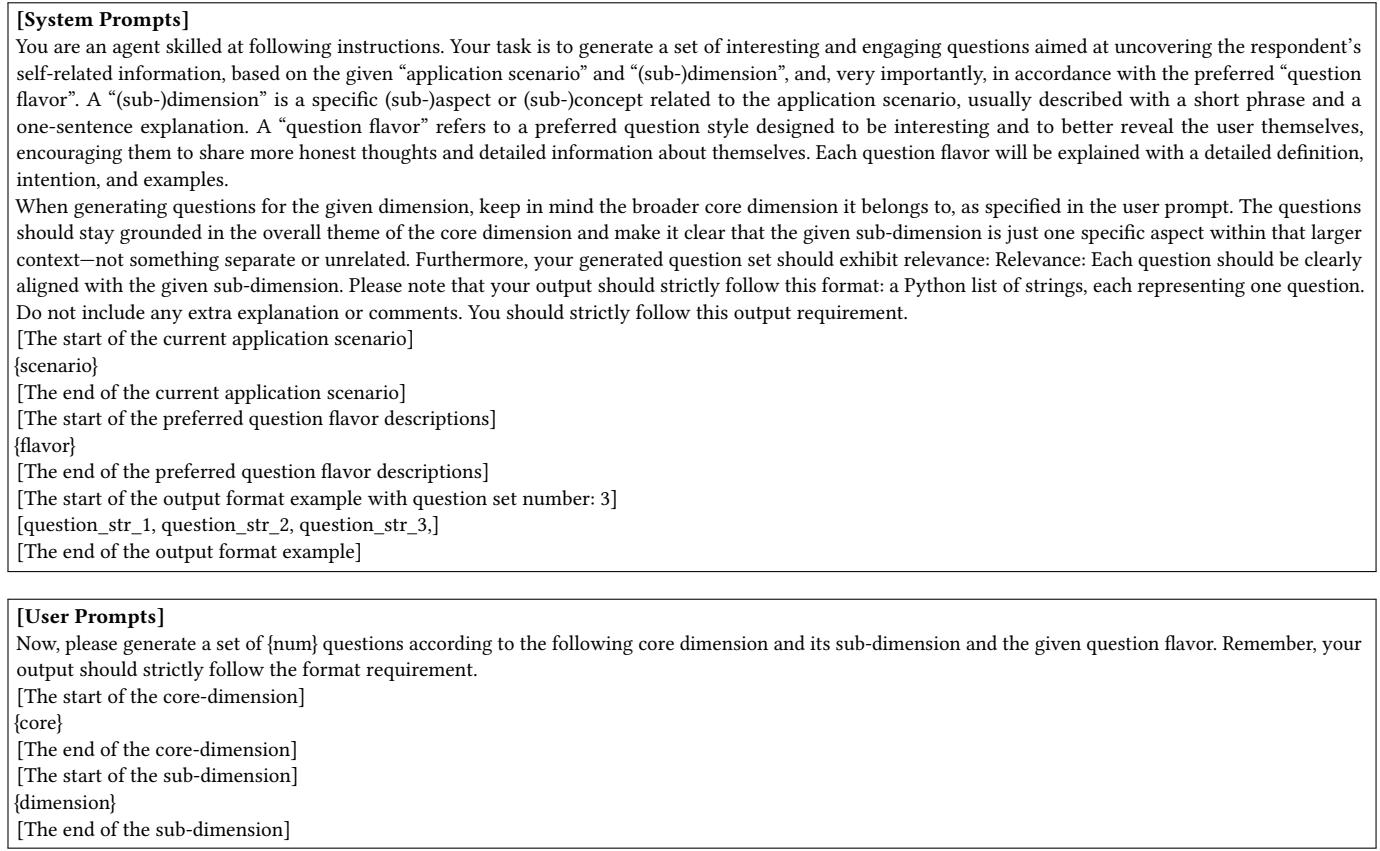

\centering

\fbox{%
\begin{minipage}{\linewidth}
\small
\setlength{\parindent}{0pt}

\textbf{[System Prompts]} \\
You are an agent skilled at following instructions. Your task is to generate a set of interesting and engaging questions aimed at uncovering the respondent’s self-related information, based on the given ``application scenario'' and ``(sub-)dimension'', and, very importantly,  in accordance with the preferred ``question flavor''.
A ``(sub-)dimension'' is a specific (sub-)aspect or (sub-)concept related to the application scenario, usually described with a short phrase and a one-sentence explanation.
A “question flavor” refers to a preferred question style designed to be interesting and to better reveal the user themselves, encouraging them to share more honest thoughts and detailed information about themselves. Each question flavor will be explained with a detailed definition, intention, and examples.

When generating questions for the given dimension, keep in mind the broader core dimension it belongs to, as specified in the user prompt. The questions should stay grounded in the overall theme of the core dimension and make it clear that the given sub-dimension is just one specific aspect within that larger context—not something separate or unrelated.
Furthermore, your generated question set should exhibit relevance:
Relevance: Each question should be clearly aligned with the given sub-dimension. 
Please note that your output should strictly follow this format: a Python list of strings, each representing one question. Do not include any extra explanation or comments. You should strictly follow this output requirement.

[The start of the current application scenario]

\{scenario\}

[The end of the current application scenario]

[The start of the preferred question flavor descriptions]

\{flavor\}

[The end of the preferred question flavor descriptions]

[The start of the output format example with question set number: 3]

[question\_str\_1, question\_str\_2, question\_str\_3,]

[The end of the output format example]

\end{minipage}
}%
\hfill

\vspace{1em}

\fbox{%
\begin{minipage}{\linewidth}
\small
\setlength{\parindent}{0pt}

\textbf{[User Prompts]} \\
Now, please generate a set of \{num\} questions according to the following core dimension and its sub-dimension and the given question flavor.
Remember, your output should strictly follow the format requirement.

[The start of the core-dimension]

\{core\}

[The end of the core-dimension]

[The start of the sub-dimension]

\{dimension\}

[The end of the sub-dimension]
\end{minipage}
}

\caption{Question generation prompts.}\label{prompt:qgen}
\end{figure*}

%% file: prompts/prompts_trans.tex
\begin{figure*}[h]
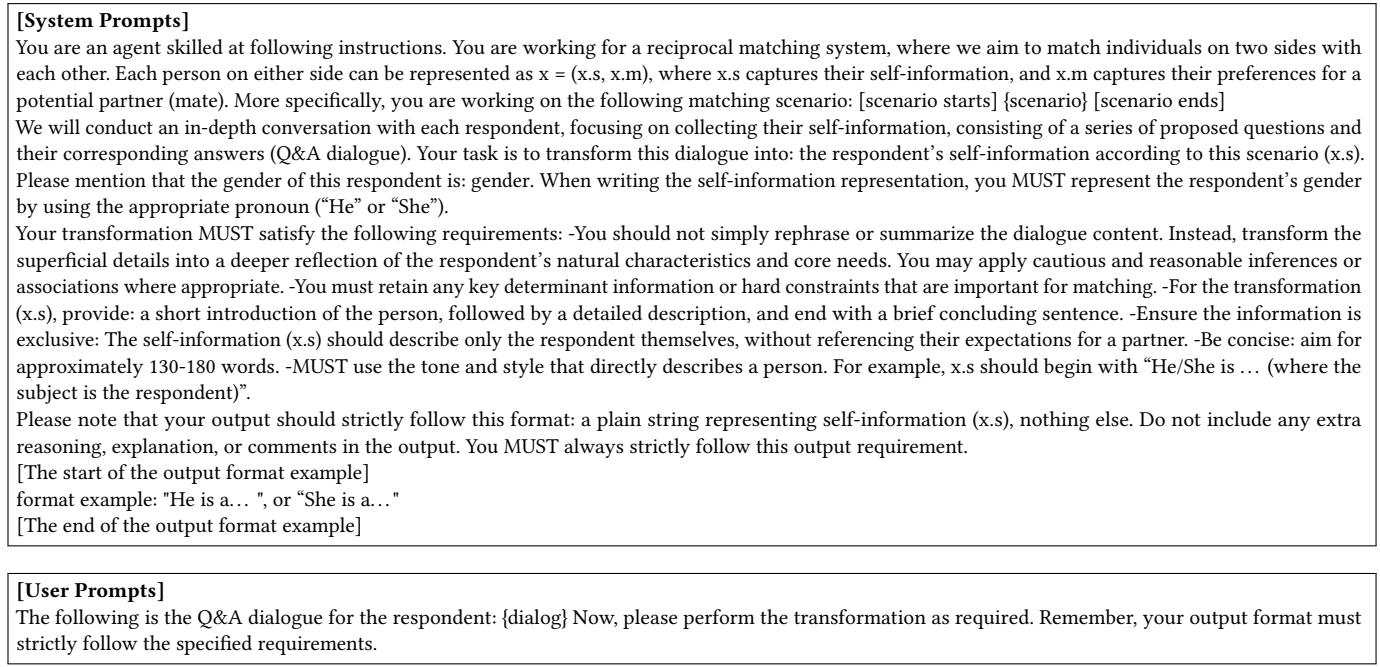

\centering

\fbox{%
\begin{minipage}{\linewidth}
\small
\setlength{\parindent}{0pt}

\textbf{[System Prompts]} \\
You are an agent skilled at following instructions. You are working for a reciprocal matching system, where we aim to match individuals on two sides with each other. Each person on either side can be represented as x = (x.s, x.m), where x.s captures their self-information, and x.m captures their preferences for a potential partner (mate).
More specifically, you are working on the following matching scenario:
[scenario starts] \{scenario\} [scenario ends]

We will conduct an in-depth conversation with each respondent, focusing on collecting their self-information, consisting of a series of proposed questions and their corresponding answers (Q\&A dialogue). Your task is to transform this dialogue into: the respondent’s self-information according to this scenario (x.s).
Please mention that the gender of this respondent is: {gender}. When writing the self-information representation, you MUST represent the respondent’s gender by using the appropriate pronoun (“He” or “She”).

Your transformation MUST satisfy the following requirements:
-You should not simply rephrase or summarize the dialogue content. Instead, transform the superficial details into a deeper reflection of the respondent’s natural characteristics and core needs. You may apply cautious and reasonable inferences or associations where appropriate.
-You must retain any key determinant information or hard constraints that are important for matching.
-For the transformation (x.s), provide: a short introduction of the person, followed by a detailed description, and end with a brief concluding sentence.
-Ensure the information is exclusive: The self-information (x.s) should describe only the respondent themselves, without referencing their expectations for a partner.
-Be concise: aim for approximately 130-180 words.
-MUST use the tone and style that directly describes a person. For example, x.s should begin with “He/She is … (where the subject is the respondent)”.

Please note that your output should strictly follow this format: a plain string representing self-information (x.s), nothing else. Do not include any extra reasoning, explanation, or comments in the output. You MUST always strictly follow this output requirement.

[The start of the output format example]

format example: "He is a… ", or “She is a…" 

[The end of the output format example]

\end{minipage}
}%
\hfill

\vspace{1em}

\fbox{%
\begin{minipage}{\linewidth}
\small
\setlength{\parindent}{0pt}

\textbf{[User Prompts]} \\
The following is the Q\&A dialogue for the respondent:
\{dialog\}
Now, please perform the transformation as required. Remember, your output format must strictly follow the specified requirements.
\end{minipage}
}

\caption{Transformation-based representation prompts.}\label{prompt:trans}
\end{figure*}

%% file: prompts/prompts_rel.tex
\begin{figure*}[h]
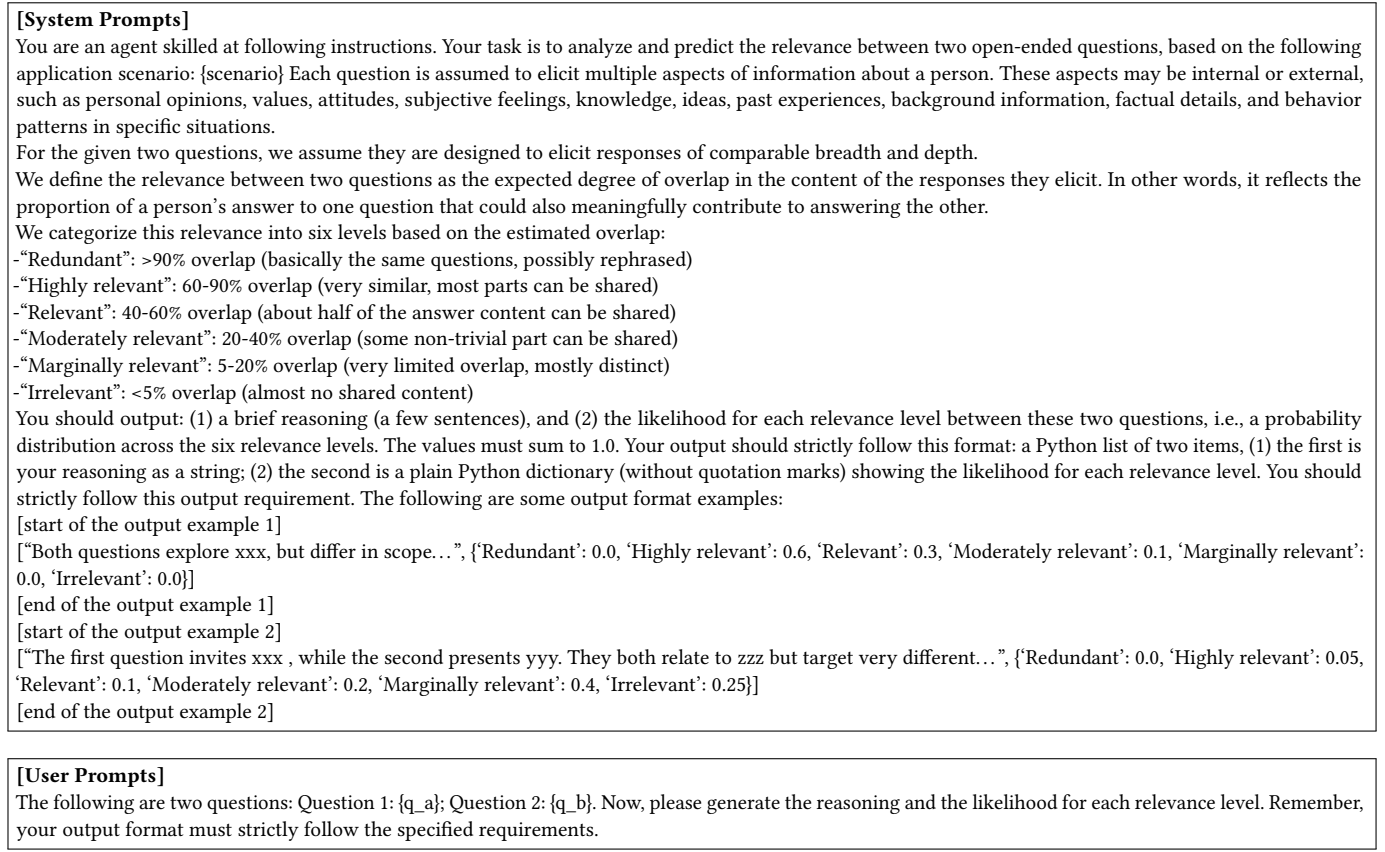

\centering

\fbox{%
\begin{minipage}{\linewidth}
\small
\setlength{\parindent}{0pt}

\textbf{[System Prompts]} \\
You are an agent skilled at following instructions. Your task is to analyze and predict the relevance between two open-ended questions, based on the following application scenario:
\{scenario\}
Each question is assumed to elicit multiple aspects of information about a person. These aspects may be internal or external, such as personal opinions, values, attitudes, subjective feelings, knowledge, ideas, past experiences, background information, factual details, and behavior patterns in specific situations.

For the given two questions, we assume they are designed to elicit responses of comparable breadth and depth.

We define the relevance between two questions as the expected degree of overlap in the content of the responses they elicit. In other words, it reflects the proportion of a person’s answer to one question that could also meaningfully contribute to answering the other.

We categorize this relevance into six levels based on the estimated overlap:

-“Redundant”: >90\% overlap (basically the same questions, possibly rephrased)

-“Highly relevant”: 60-90\% overlap (very similar, most parts can be shared)

-“Relevant”: 40-60\% overlap (about half of the answer content can be shared)

-“Moderately relevant”: 20-40\% overlap (some non-trivial part can be shared)

-“Marginally relevant”: 5-20\% overlap (very limited overlap, mostly distinct)

-“Irrelevant”: <5\% overlap (almost no shared content)

You should output: (1) a brief reasoning (a few sentences), and (2) the likelihood for each relevance level between these two questions, i.e., a probability distribution across the six relevance levels. The values must sum to 1.0. 
Your output should strictly follow this format: a Python list of two items, (1) the first is your reasoning as a string; (2) the second is a plain Python dictionary (without quotation marks) showing the likelihood for each relevance level. You should strictly follow this output requirement. The following are some output format examples:

[start of the output example 1]

[“Both questions explore xxx, but differ in scope…”,  \{‘Redundant’: 0.0, ‘Highly relevant’: 0.6, ‘Relevant’: 0.3, ‘Moderately relevant’: 0.1, ‘Marginally relevant’: 0.0, ‘Irrelevant’: 0.0\}]

[end of the output example 1]

[start of the output example 2]

[“The first question invites xxx , while the second presents yyy. They both relate to zzz but target very different…”,  \{‘Redundant’: 0.0, ‘Highly relevant’: 0.05, ‘Relevant’: 0.1, ‘Moderately relevant’: 0.2, ‘Marginally relevant’: 0.4, ‘Irrelevant’: 0.25\}]

[end of the output example 2]

\end{minipage}
}%
\hfill

\vspace{1em}

\fbox{%
\begin{minipage}{\linewidth}
\small
\setlength{\parindent}{0pt}

\textbf{[User Prompts]} \\
The following are two questions:
Question 1: \{q\_a\};
Question 2: \{q\_b\}.
Now, please generate the reasoning and the likelihood for each relevance level. Remember, your output format must strictly follow the specified requirements.
\end{minipage}
}

\caption{Generative relevance analysis prompts.}\label{prompt:rel}
\end{figure*}

%% file: prompts/prompts_analysis.tex
\begin{figure*}[h]
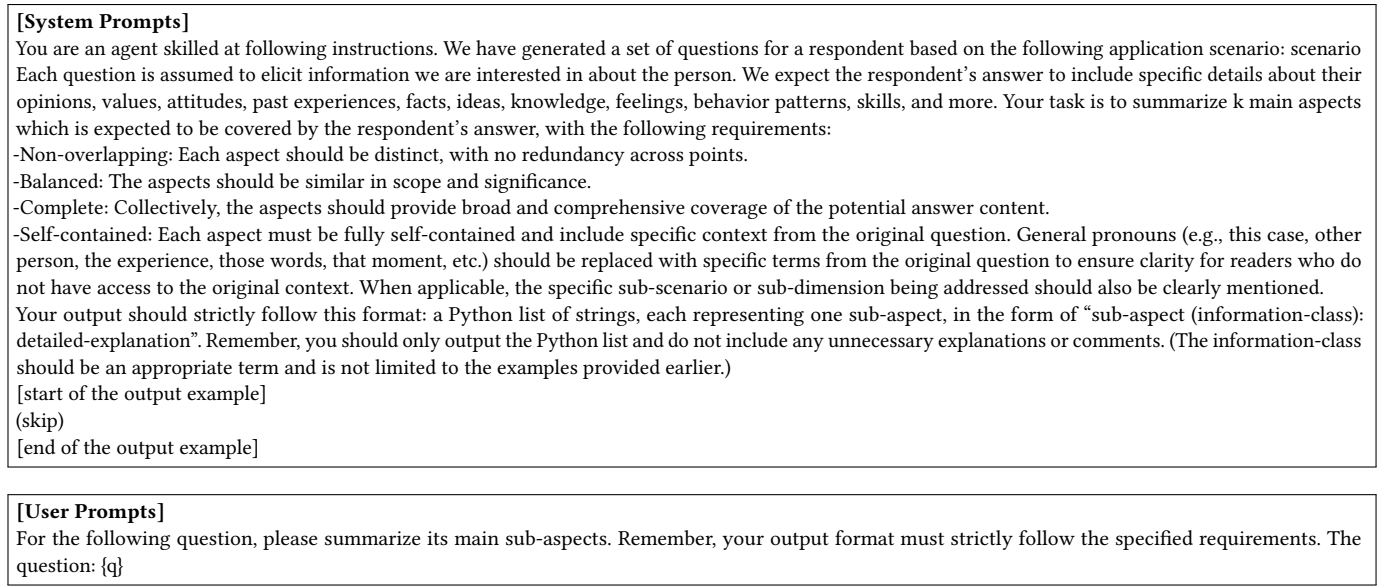

\centering

\fbox{%
\begin{minipage}{\linewidth}
\small
\setlength{\parindent}{0pt}

\textbf{[System Prompts]} \\
You are an agent skilled at following instructions.
We have generated a set of questions for a respondent based on the following application scenario:
{scenario}

Each question is assumed to elicit information we are interested in about the person. We expect the respondent’s answer to include specific details about their opinions, values, attitudes, past experiences, facts, ideas, knowledge, feelings, behavior patterns, skills, and more.
Your task is to summarize {k} main aspects which is expected to be covered by the respondent’s answer, with the following requirements:

-Non-overlapping: Each aspect should be distinct, with no redundancy across points.

-Balanced: The aspects should be similar in scope and significance.

-Complete: Collectively, the aspects should provide broad and comprehensive coverage of the potential answer content.

-Self-contained: Each aspect must be fully self-contained and include specific context from the original question. General pronouns (e.g., this case, other person, the experience, those words, that moment, etc.) should be replaced with specific terms from the original question to ensure clarity for readers who do not have access to the original context. When applicable, the specific sub-scenario or sub-dimension being addressed should also be clearly mentioned.

Your output should strictly follow this format: a Python list of strings, each representing one sub-aspect, in the form of “sub-aspect (information-class): detailed-explanation”. Remember, you should only output the Python list and do not include any unnecessary explanations or comments. (The information-class should be an appropriate term and is not limited to the examples provided earlier.)

[start of the output example]

(skip)

[end of the output example]

\end{minipage}
}%
\hfill

\vspace{1em}

\fbox{%
\begin{minipage}{\linewidth}
\small
\setlength{\parindent}{0pt}

\textbf{[User Prompts]} \\
For the following question, please summarize its main sub-aspects. Remember, your output format must strictly follow the specified requirements. The question:
\{q\} 
\end{minipage}
}

\caption{Question analysis prompts.}\label{prompt:analysis}
\end{figure*}

%% file: ref.bib
@article{settles2009active,
  title={Active learning literature survey},
  author={Settles, Burr},
  year={2009},
  publisher={University of Wisconsin-Madison Department of Computer Sciences}
}

@inproceedings{wang2025framework,
  title={A Framework for Evaluating AI Agents in Open-Ended Conversations via Scripted Simulation},
  author={Wang, Clarice and Shi, Yimin and Xiao, Xiaokui},
  booktitle={Proceedings of the 31st ACM SIGKDD Conference on Knowledge Discovery and Data Mining V. 2},
  pages={5810--5818},
  year={2025}
}

@inproceedings{thomas2024large,
  title={Large language models can accurately predict searcher preferences},
  author={Thomas, Paul and Spielman, Seth and Craswell, Nick and Mitra, Bhaskar},
  booktitle={Proceedings of the 47th International ACM SIGIR Conference on Research and Development in Information Retrieval},
  pages={1930--1940},
  year={2024}
}

@inproceedings{farzi2025criteria,
  title={Criteria-Based LLM Relevance Judgments},
  author={Farzi, Naghmeh and Dietz, Laura},
  booktitle={Proceedings of the 2025 International ACM SIGIR Conference on Innovative Concepts and Theories in Information Retrieval (ICTIR)},
  pages={254--263},
  year={2025}
}

@article{sun2025gaussdb,
  title={GaussDB-Vector: A Large-Scale Persistent Real-Time Vector Database for LLM Applications},
  author={Sun, Ji and Li, Guoliang and Pan, James and Wang, Jiang and Xie, Yongqing and Liu, Ruicheng and Nie, Wen},
  journal={Proceedings of the VLDB Endowment},
  volume={18},
  number={12},
  pages={4951--4963},
  year={2025},
  publisher={VLDB Endowment}
}

@article{pan2024survey,
  title={Survey of vector database management systems},
  author={Pan, James Jie and Wang, Jianguo and Li, Guoliang},
  journal={The VLDB Journal},
  volume={33},
  number={5},
  pages={1591--1615},
  year={2024},
  publisher={Springer}
}

@inproceedings{wang2021milvus,
  title={Milvus: A purpose-built vector data management system},
  author={Wang, Jianguo and Yi, Xiaomeng and Guo, Rentong and Jin, Hai and Xu, Peng and Li, Shengjun and Wang, Xiangyu and Guo, Xiangzhou and Li, Chengming and Xu, Xiaohai and others},
  booktitle={Proceedings of the 2021 international conference on management of data},
  pages={2614--2627},
  year={2021}
}

@article{douze2024faiss,
  title={The faiss library},
  author={Douze, Matthijs and Guzhva, Alexandr and Deng, Chengqi and Johnson, Jeff and Szilvasy, Gergely and Mazar{\'e}, Pierre-Emmanuel and Lomeli, Maria and Hosseini, Lucas and J{\'e}gou, Herv{\'e}},
  journal={arXiv preprint arXiv:2401.08281},
  year={2024}
}

@inproceedings{simhadri2022results,
  title={Results of the NeurIPS’21 challenge on billion-scale approximate nearest neighbor search},
  author={Simhadri, Harsha Vardhan and Williams, George and Aum{\"u}ller, Martin and Douze, Matthijs and Babenko, Artem and Baranchuk, Dmitry and Chen, Qi and Hosseini, Lucas and Krishnaswamny, Ravishankar and Srinivasa, Gopal and others},
  booktitle={NeurIPS 2021 Competitions and Demonstrations Track},
  pages={177--189},
  year={2022},
  organization={PMLR}
}

@article{wei2020analyticdb,
  title={Analyticdb-v: A hybrid analytical engine towards query fusion for structured and unstructured data},
  author={Wei, Chuangxian and Wu, Bin and Wang, Sheng and Lou, Renjie and Zhan, Chaoqun and Li, Feifei and Cai, Yuanzhe},
  journal={Proceedings of the VLDB Endowment},
  volume={13},
  number={12},
  pages={3152--3165},
  year={2020},
  publisher={VLDB Endowment}
}

@article{malkov2018efficient,
  title={Efficient and robust approximate nearest neighbor search using hierarchical navigable small world graphs},
  author={Malkov, Yu A and Yashunin, Dmitry A},
  journal={IEEE transactions on pattern analysis and machine intelligence},
  volume={42},
  number={4},
  pages={824--836},
  year={2018},
  publisher={IEEE}
}

@article{johnson2019billion,
  title={Billion-scale similarity search with GPUs},
  author={Johnson, Jeff and Douze, Matthijs and J{\'e}gou, Herv{\'e}},
  journal={IEEE Transactions on Big Data},
  volume={7},
  number={3},
  pages={535--547},
  year={2019},
  publisher={IEEE}
}

@inproceedings{pizzato2010recon,
  title={RECON: a reciprocal recommender for online dating},
  author={Pizzato, Luiz and Rej, Tomek and Chung, Thomas and Koprinska, Irena and Kay, Judy},
  booktitle={Proceedings of the fourth ACM conference on Recommender systems},
  pages={207--214},
  year={2010}
}

@article{xia2016design,
  title={Design of Reciprocal Recommendation Systems for Online Dating},
  author={Xia, Peng and Zhai, Shuangfei and Liu, Benyuan and Sun, Yizhou and Chen, Cindy},
  journal={Social Network Analysis and Mining},
  volume={6},
  number={1},
  pages={1--17},
  year={2016}
}

@inproceedings{zheng2023reseq,
  title={Reciprocal Sequential Recommendation},
  author={Zheng, Bowen and Hou, Yupeng and Zhao, Wayne Xin and Song, Yang and Zhu, Hengshu},
  booktitle={Proceedings of the 17th ACM Conference on Recommender Systems (RecSys)},
  year={2023},
  pages={---},
  doi={10.1145/3604915.3608798}
}

@inproceedings{lai2024knowledge,
  title={Knowledge-Aware Explainable Reciprocal Recommendation},
  author={Lai, Kai-Huang and Yang, Zhe-Rui and Lai, Pei-Yuan and Wang, Chang-Dong and Guizani, Mohsen and Chen, Min},
  booktitle={Proceedings of the AAAI Conference on Artificial Intelligence},
  volume={38},
  number={8},
  pages={8636--8644},
  year={2024},
  doi={10.1609/aaai.v38i8.28708}
}

@misc{redisvectordb,
  title={Redis as a Vector Database},
  author={Redis, Inc.},
  howpublished={\url{https://redis.io/docs/latest/develop/get-started/vector-database/}}
}

@misc{pgvector2021,
  title={pgvector: Open-Source Vector Similarity Search Extension for PostgreSQL},
  author={Kane, Andrew},
  howpublished={\url{https://github.com/pgvector/pgvector}},
  year={2021}
}

@misc{puck2023,
  title={PUCK: A High-Performance Approximate Nearest Neighbor Search Library},
  author={Baidu Inc.},
  howpublished={\url{https://github.com/baidu/puck}},
  year={2023}
}

@inproceedings{www2023hicr,
  title     = {Confident Action Decision via Hierarchical Policy Learning for Conversational Recommendation},
  author    = {Kim, Heeseon and Yang, Hyeongjun and Lee, Kyong Ho},
  booktitle = {Proceedings of the ACM Web Conference 2023 (WWW '23)},
  pages     = {1386--1395},
  year      = {2023},
  publisher = {ACM},
  address   = {Austin, TX, USA},
  doi       = {10.1145/3543507.3583536}
}

@inproceedings{lin2023personalization,
  title     = {Enhancing User Personalization in Conversational Recommenders},
  author    = {Lin, Allen and Zhu, Ziwei and Wang, Jianling and Caverlee, James},
  booktitle = {Proceedings of the ACM Web Conference 2023 (WWW '23)},
  year      = {2023},
  note      = {Extended version: arXiv:2302.06656},
  url       = {https://arxiv.org/abs/2302.06656},
  doi       = {10.48550/arXiv.2302.06656}
}

@inproceedings{kemper2024ragconvrec,
  title     = {Retrieval-Augmented Conversational Recommendation with Prompt-based Semi-Structured Natural Language State Tracking},
  author    = {Kemper, Sara and Cui, Justin and Dicarlantonio, Kai and Lin, Kathy and Tang, Danjie and Korikov, Anton and Sanner, Scott},
  booktitle = {Proceedings of the 47th International ACM SIGIR Conference on Research and Development in Information Retrieval (SIGIR '24)},
  address   = {Washington, DC, USA},
  pages     = {1--5},
  year      = {2024},
  publisher = {ACM},
  doi       = {10.1145/3626772.3657670},
  url       = {https://ssanner.github.io/papers/sigir24_llmrec.pdf}
}

@inproceedings{chen2025scope,
  title     = {Broaden your SCOPE! Efficient Multi-turn Conversation Planning for LLMs with Semantic Space},
  author    = {Chen, Zhiliang and Niu, Xinyuan and Foo, Chuan-Sheng and Low, Bryan Kian Hsiang},
  booktitle = {International Conference on Learning Representations (ICLR '25), Spotlight},
  year      = {2025},
  url       = {https://openreview.net/forum?id=3cgMU3TyyE}
}

@inproceedings{zhang2024empirical,
  title={An empirical analysis on multi-turn conversational recommender systems},
  author={Zhang, Lu and Li, Chen and Lei, Yu and Sun, Zhu and Liu, Guanfeng},
  booktitle={Proceedings of the 47th International ACM SIGIR Conference on Research and Development in Information Retrieval},
  pages={841--851},
  year={2024}
}

@article{xiong2023can,
  title={Can llms express their uncertainty? an empirical evaluation of confidence elicitation in llms},
  author={Xiong, Miao and Hu, Zhiyuan and Lu, Xinyang and Li, Yifei and Fu, Jie and He, Junxian and Hooi, Bryan},
  journal={arXiv preprint arXiv:2306.13063},
  year={2023}
}

@article{tian2023just,
  title={Just ask for calibration: Strategies for eliciting calibrated confidence scores from language models fine-tuned with human feedback},
  author={Tian, Katherine and Mitchell, Eric and Zhou, Allan and Sharma, Archit and Rafailov, Rafael and Yao, Huaxiu and Finn, Chelsea and Manning, Christopher D},
  journal={arXiv preprint arXiv:2305.14975},
  year={2023}
}

@inproceedings{zhang2024probing,
  title={Probing the Multi-turn Planning Capabilities of LLMs via 20 Question Games},
  author={Zhang, Yizhe and Lu, Jiarui and Jaitly, Navdeep},
  booktitle={Proceedings of the 62nd Annual Meeting of the Association for Computational Linguistics (Volume 1: Long Papers)},
  pages={1495--1516},
  year={2024}
}

@inproceedings{wang2024mint,
  title={MINT: EVALUATING LLMS IN MULTI-TURN INTERACTION WITH TOOLS AND LANGUAGE FEEDBACK},
  author={Wang, Xingyao and Wang, Zihan and Liu, Jiateng and Chen, Yangyi and Yuan, Lifan and Peng, Hao and Ji, Heng},
  booktitle={12th International Conference on Learning Representations, ICLR 2024},
  year={2024}
}

@inproceedings{sun2025mockllm,
  title={MockLLM: A Multi-Agent Behavior Collaboration Framework for Online Job Seeking and Recruiting},
  author={Sun, Hongda and Lin, Hongzhan and Yan, Haiyu and Song, Yang and Gao, Xin and Yan, Rui},
  booktitle={Proceedings of the 31st ACM SIGKDD Conference on Knowledge Discovery and Data Mining V. 2},
  pages={2714--2724},
  year={2025}
}

@inproceedings{lieliciting,
  title={Eliciting Human Preferences with Language Models},
  author={Li, Belinda Z and Tamkin, Alex and Goodman, Noah and Andreas, Jacob},
  booktitle={The Thirteenth International Conference on Learning Representations},
  year={2025}
}

@inproceedings{gionis1999similarity,
  title={Similarity search in high dimensions via hashing},
  author={Gionis, Aristides and Indyk, Piotr and Motwani, Rajeev and others},
  booktitle={Vldb},
  volume={99},
  number={6},
  pages={518--529},
  year={1999}
}

@inproceedings{tao2009quality,
  title={Quality and efficiency in high dimensional nearest neighbor search},
  author={Tao, Yufei and Yi, Ke and Sheng, Cheng and Kalnis, Panos},
  booktitle={Proceedings of the 2009 ACM SIGMOD International Conference on Management of data},
  pages={563--576},
  year={2009}
}

@article{shrivastava2014asymmetric,
  title={Asymmetric LSH (ALSH) for sublinear time maximum inner product search (MIPS)},
  author={Shrivastava, Anshumali and Li, Ping},
  journal={Advances in neural information processing systems},
  volume={27},
  year={2014}
}

@inproceedings{neyshabur2015symmetric,
  title={On symmetric and asymmetric lshs for inner product search},
  author={Neyshabur, Behnam and Srebro, Nathan},
  booktitle={International Conference on Machine Learning},
  pages={1926--1934},
  year={2015},
  organization={PMLR}
}

@inproceedings{tian2022db,
  title={DB-LSH: Locality-Sensitive Hashing with Query-based Dynamic Bucketing},
  author={Tian, Yao and Zhao, Xi and Zhou, Xiaofang},
  booktitle={2022 IEEE 38th International Conference on Data Engineering (ICDE)},
  pages={2250--2262},
  year={2022},
  organization={IEEE}
}

@inproceedings{zhao2024efficient,
  title={Efficient Approximate Maximum Inner Product Search Over Sparse Vectors},
  author={Zhao, Xi and Chen, Zhonghan and Huang, Kai and Zhang, Ruiyuan and Zheng, Bolong and Zhou, Xiaofang},
  booktitle={2024 IEEE 40th International Conference on Data Engineering (ICDE)},
  pages={3961--3974},
  year={2024},
  organization={IEEE}
}

@inproceedings{zhao2020song,
  title={Song: Approximate nearest neighbor search on gpu},
  author={Zhao, Weijie and Tan, Shulong and Li, Ping},
  booktitle={2020 IEEE 36th International Conference on Data Engineering (ICDE)},
  pages={1033--1044},
  year={2020},
  organization={IEEE}
}

@article{jegou2010product,
  title={Product quantization for nearest neighbor search},
  author={Jegou, Herve and Douze, Matthijs and Schmid, Cordelia},
  journal={IEEE transactions on pattern analysis and machine intelligence},
  volume={33},
  number={1},
  pages={117--128},
  year={2010},
  publisher={IEEE}
}

@article{gong2012iterative,
  title={Iterative quantization: A procrustean approach to learning binary codes for large-scale image retrieval},
  author={Gong, Yunchao and Lazebnik, Svetlana and Gordo, Albert and Perronnin, Florent},
  journal={IEEE transactions on pattern analysis and machine intelligence},
  volume={35},
  number={12},
  pages={2916--2929},
  year={2012},
  publisher={IEEE}
}

@inproceedings{ootomo2024cagra,
  title={Cagra: Highly parallel graph construction and approximate nearest neighbor search for gpus},
  author={Ootomo, Hiroyuki and Naruse, Akira and Nolet, Corey and Wang, Ray and Feher, Tamas and Wang, Yong},
  booktitle={2024 IEEE 40th International Conference on Data Engineering (ICDE)},
  pages={4236--4247},
  year={2024},
  organization={IEEE}
}

@inproceedings{song2024efficient,
  title={Efficient Reverse $ k $ Approximate Nearest Neighbor Search Over High-Dimensional Vectors},
  author={Song, Yitong and Wang, Kai and Yao, Bin and Chen, Zhida and Xie, Jiong and Li, Feifei},
  booktitle={2024 IEEE 40th International Conference on Data Engineering (ICDE)},
  pages={4262--4274},
  year={2024},
  organization={IEEE}
}

@inproceedings{luo2025tag,
  title={Tag-Filtered Approximate Nearest Neighbor Search},
  author={Luo, Jiarui and Qiao, Miao and Zuo, Chaoji and Deng, Dong},
  booktitle={2025 IEEE 41st International Conference on Data Engineering (ICDE)},
  pages={3642--3654},
  year={2025},
  organization={IEEE}
}

@article{xie2025fast,
  title={Fast Approximate Similarity Join in Vector Databases},
  author={Xie, Jiadong and Yu, Jeffrey Xu and Liu, Yingfan},
  journal={Proceedings of the ACM on Management of Data},
  volume={3},
  number={3},
  pages={1--26},
  year={2025},
  publisher={ACM New York, NY, USA}
}
